\documentclass[11pt,twoside]{article}

\usepackage{fullpage}

\usepackage{epsf}
\usepackage{graphics}
\usepackage{subcaption}
\usepackage{psfrag}

\usepackage{color}

\usepackage{mathtools}
\usepackage{amsfonts}
\usepackage{amsthm}
\usepackage{amsmath}
\usepackage{amssymb}

\newtheorem{theorem}{Theorem}
\newtheorem{proposition}{Proposition}
\newtheorem{lemma}{Lemma}

\long\def\comment#1{}

\newcommand{\defn}{:\,=}

\newcommand{\id}{\mathsf{id}}

\newcommand{\dimone}{n_1}
\newcommand{\dimtwo}{n_2}
\newcommand{\maxdimonetwo}{\left(\dimone \lor \dimtwo \right)}

\newcommand{\csst}{\mathbb{C}_{\mathsf{SST}}}
\newcommand{\cdiff}{\mathbb{C}_{\mathsf{DIFF}}}

\newcommand{\Cbiso}{\mathbb{C}_{\mathsf{BISO}}}

\newcommand{\bl}{\mathsf{bl}}

\newcommand{\pobs}{p_{\mathsf{obs}}}

\newcommand{\sgn}{\mathsf{sgn}}

\newcommand{\E}{\ensuremath{{\mathbb{E}}}}

\newcommand{\1}{\ensuremath{{\sf (i)}}}
\newcommand{\2}{\ensuremath{{\sf (ii)}}}
\newcommand{\3}{\ensuremath{{\sf (iii)}}}

\DeclareMathOperator{\modd}{mod}

\newcommand{\BER}{\ensuremath{\mbox{\sf Ber}}}
\newcommand{\BIN}{\ensuremath{\mbox{\sf Bin}}}

\newcommand{\thetastar}{\ensuremath{\theta^*}}
\newcommand{\thetahat}{\ensuremath{\widehat{\theta}}}

\newcommand{\Mhat}{\ensuremath{\widehat{M}}}
\newcommand{\Mstar}{\ensuremath{M^*}}
\newcommand{\Mtilde}{\ensuremath{\widetilde{M}}}

\newcommand{\CRL}{\ensuremath{\mathsf{CRL}}}

\newcommand{\Mhatls}{\ensuremath{\widehat{M}}_{{\sf LS}}}

\newcommand{\real}{\ensuremath{\mathbb{R}}}

\newcommand{\pihat}{\ensuremath{\widehat{\pi}}}

\newcommand{\pihattds}{\ensuremath{\widehat{\pi}_{\sf tds}}}

\newcommand{\pihatftds}{\ensuremath{\widehat{\pi}_{{\sf tds}}}}
\newcommand{\sigmahat}{\ensuremath{\widehat{\sigma}}}

\newcommand{\sigmahattds}{\ensuremath{\widehat{\sigma}_{\sf tds}}}
\newcommand{\sigmahatftds}{\ensuremath{\widehat{\sigma}_{\sf tds}}}
\newcommand{\sigmahatpre}{\ensuremath{\widehat{\sigma}_{\sf pre}}}

\newcommand{\EE}{\ensuremath{\mathbb{E}}}

\newcommand{\symgp}{\mathfrak{S}}
\newcommand{\vars}{\zeta}
\newcommand{\maxvarone}{(\vars^2 \lor 1)}
\newcommand{\varplusone}{(\vars + 1)}

\newcommand{\bfone}{\mathbf{1}}
\DeclareMathOperator{\trace}{\mathsf{tr}}
\DeclareMathOperator{\Poi}{\mathsf{Poi}}
\DeclareMathOperator{\var}{\mathsf{var}}

\newcommand{\Cperm}{\mathbb{C}_{\mathsf{Perm}}}
\newcommand{\Cskew}{\mathbb{C}_{\mathsf{skew}}}

\newcommand{\BL}{\mathsf{BL}}

\newcommand{\Mps}{M_{\pi, \sigma}}
\newcommand{\Dps}{\Delta_{\pi, \sigma}}
\newcommand{\Cdiffps}{\cdiff(\pi, \sigma)}
\newcommand{\Zps}{Z_{\pi, \sigma}}
\newcommand{\Ztil}{\Zps}
\newcommand{\At}{\mathcal{A}_t}
\newcommand{\Wtil}{\widetilde{W}}
\newcommand{\cE}{\mathcal{E}}

\newcommand{\bllarge}{\mathsf{BL}^{\mathbb{L}}}
\newcommand{\blsmall}{\mathsf{BL}^{\mathbb{S}}}
\newcommand{\blone}{\BL^{(1)}}
\newcommand{\bltwo}{\BL^{(2)}}
\newcommand{\blt}{\BL^{(t)}}

\newcommand{\blocksize}{\dimtwo \sqrt{\frac{\dimone}{N} \log (\dimone \dimtwo) }}

\newcommand\blfootnote[1]{%
  \begingroup
  \renewcommand\thefootnote{}\footnote{#1}%
  \addtocounter{footnote}{-1}%
  \endgroup
}

\newcommand{\citep}{\cite}
\newcommand{\citet}{\cite}

\newcommand{\order}{\ensuremath{\mathcal{O}}}
\newcommand{\ordertil}{\ensuremath{\widetilde{\order}}}

\begin{document}

\begin{center}

{\bf{\LARGE{Breaking the $1/\sqrt{n}$ Barrier: Faster Rates for
      \\ Permutation-based Models in Polynomial Time}}}

\vspace*{.2in}

{\large{
\begin{tabular}{ccc}
Cheng Mao$^\star$ & Ashwin Pananjady$^\dagger$ & Martin
J. Wainwright$^{\dagger, \ddagger}$
\end{tabular}
}}
\vspace*{.2in}

\begin{tabular}{c}
Department of Mathematics, MIT$^\star$ \\
Department of Electrical Engineering and Computer Sciences, UC Berkeley$^\dagger$ \\
Department of Statistics, UC Berkeley$^\ddagger$
\end{tabular}

\vspace*{.2in}

\today

\end{center}
\vspace*{.2in}

\begin{abstract}
  Many applications, including rank aggregation and crowd-labeling,
  can be modeled in terms of a bivariate isotonic matrix with unknown
  permutations acting on its rows and columns.  We consider the
  problem of estimating such a matrix based on noisy observations of a
  subset of its entries, and design and analyze a polynomial-time
  algorithm that improves upon the state of the art. In particular,
  our results imply that any such $n \times n$ matrix can be estimated
  efficiently in the normalized Frobenius norm at rate
  $\widetilde{\mathcal O}(n^{-3/4})$, thus narrowing the gap between
  $\widetilde{\mathcal O}(n^{-1})$ and $\widetilde{\mathcal
    O}(n^{-1/2})$, which were hitherto the rates of the most
  statistically and computationally efficient methods, respectively.
\end{abstract}


\section{Introduction}

Structured\blfootnote{Accepted for presentation at Conference on Learning Theory (COLT) 2018} matrices with entries in the range $[0, 1]$ and unknown
permutations acting on their rows and columns arise in multiple
applications, including estimation from pairwise comparisons
\citep{BraTer52,ShaBalGunWai17} and crowd-labeling
\citep{DawSke79,ShaBalWai16}. Traditional parametric models~\citet{BraTer52,Luc59,Thu27,DawSke79} assume that these
matrices are obtained from rank-one matrices via a known link
function. Aided by tools such as maximum likelihood estimation and
spectral methods, researchers have made significant progress in
studying both statistical and computational aspects of these
parametric
models~\citep{HajOhXu14,RajAga14,Shaetal16,NegOhSha16,ZhaCheDenJor16,
  GaoZho13, GaoLuZho16, KarOhSha11-b, LiuPenIhl12, DasDasRas13,
  GhoKalMcA11} and their low-rank
generalizations~\citep{RajAga16,NegOhTheXu17,
  KarOhSha11}.~\nocite{LeeSha17, CheSuh15, ParNeeZhaSanDhi15}

There has been evidence from empirical studies
(e.g.,~\cite{McLLuc65,BalWil97}) that real-world data is not always
well-captured by such parametric models.  With the goal of increasing
model flexibility, a recent line of work has studied the class of
\emph{permutation-based} models~\citep{Cha15,
  ShaBalGunWai17,ShaBalWai16}.  Rather than imposing parametric
conditions on the matrix entries, these models impose only shape
constraints on the matrix, such as monotonicity, before unknown
permutations act on the its rows and columns.  This more flexible
class reduces modeling bias compared to its parametric counterparts
while, perhaps surprisingly, producing models that can be estimated at
rates that differ only by logarithmic factors from parametric models. On
the negative side, these advantages of permutation-based models are
accompanied by significant computational challenges. The unknown
permutations make the parameter space highly non-convex, so that
efficient maximum likelihood estimation is unlikely. Moreover,
spectral methods are often suboptimal in approximating
shape-constrained sets of
matrices~\citep{Cha15,ShaBalGunWai17}. Consequently, results from many
recent papers show a non-trivial statistical-computational gap in
estimation rates for models with latent
permutations~\citep{ShaBalGunWai17,ChaMuk16,ShaBalWai16,FlaMaoRig16,PanWaiCou17}.

\paragraph{Related work.}
While the main motivation of our work comes from nonparametric methods
for aggregating pairwise comparisons, we begin by discussing a few
other lines of related work. The current paper lies at the
intersection of shape-constrained estimation and latent permutation
learning. Shape-constrained estimation has long been a major topic in
nonparametric statistics, and of particular relevance to our work is
the estimation of a bivariate isotonic matrix without latent
permutations~\citep{ChaGunSen18}. There, it was shown that the minimax
rate of estimating an $n \times n$ matrix from noisy observations of
all its entries is $\widetilde \Theta(n^{-1})$. The upper bound is
achieved by the least squares estimator, which is efficiently
computable due to the convexity of the parameter space.

Shape-constrained matrices with permuted rows or columns also arise in
applications such as seriation~\citep{FogJenBacdAs13,FlaMaoRig16} and
feature matching~\citep{ColDal16}. In particular, the monotone
subclass of the statistical seriation model~\citep{FlaMaoRig16}
contains $n \times n$ matrices that have increasing columns, and an
unknown row permutation. The authors established the minimax rate
$\widetilde \Theta(n^{-2/3})$ for estimating matrices in this class
and proposed a computationally efficient algorithm with rate
$\ordertil (n^{-1/2})$. For the subclass of such matrices where in
addition, the rows are also monotone, the results of the current paper
improve the two rates to $\ordertil (n^{-1})$ and $\ordertil (n^{-3/4})$ respectively.

Another related model is that of noisy sorting~\citep{BraMos08}, which
involves a latent permutation but no shape-constraint. In this
prototype of a permutation-based ranking model, we have an unknown, $n
\times n$ matrix with constant upper and lower triangular portions
whose rows and columns are acted upon by an unknown permutation. The
hardness of recovering any such matrix in noise lies in estimating the
unknown permutation. As it turns out, this class of matrices can be
estimated efficiently at minimax optimal rate $\widetilde
\Theta(n^{-1})$ by multiple procedures: the original work by Braverman
and Mossel~\citet{BraMos08} proposed an algorithm with time complexity
$\order(n^c)$ for some unknown and large constant $c$, and recently, an
$\ordertil(n^2)$-time algorithm was proposed by~Mao et
al.~\citet{MaoWeeRig17}.  These algorithms, however, do not generalize
beyond the noisy sorting class, which constitutes a small subclass of
an interesting class of matrices that we describe next.

The most relevant body of work to the current paper is that on
estimating matrices satisfying the \emph{strong stochastic
  transitivity} condition, or SST for short. This class of matrices
contains all $n \times n$ bivariate isotonic matrices with unknown
permutations acting on their rows and columns, with an additional
skew-symmetry constraint. The first theoretical study of these
matrices was carried out by Chatterjee~\citet{Cha15}, who showed that
a spectral algorithm achieved the rate $\ordertil(n^{-1/4})$ in the
normalized Frobenius norm. Shah et al.~\citet{ShaBalGunWai17} then
showed that the minimax rate of estimation is given by $\widetilde
\Theta (n^{-1})$, and also improved the analysis of the spectral
estimator of Chatterjee~\citet{Cha15} to obtain the computationally
efficient rate $\ordertil (n^{-1/2})$. In follow-up
work~\citep{ShaBalWai16-2}, they also showed a second $\CRL$ estimator
based on the Borda count that achieved the same rate, but in
near-linear time. In related work, Chatterjee and
Mukherjee~\citet{ChaMuk16} analyzed a variant of the $\CRL$ estimator,
showing that for subclasses of SST matrices, it achieved rates that
were faster than $\order(n^{-1/2})$. In a complementary direction, a
superset of the current authors~\citet{PanMaoMutWaiCou17} analyzed the
estimation problem under an observation model with structured missing
data, and showed that for many observation patterns, a variant of the
$\CRL$ estimator was minimax optimal.

Shah et al.~\citet{ShaBalWai16-2} also showed that conditioned on the
planted clique conjecture, it is impossible to improve upon a certain
notion of adaptivity of the $\CRL$ estimator in polynomial time.  Such
results have prompted various authors~\citep{FlaMaoRig16,
  ShaBalWai16-2} to conjecture that a similar
statistical-computational gap also exists when estimating SST matrices
in the Frobenius norm.


\paragraph{Our contributions.}
Our main contribution in the current work is to tighten the
aforementioned statistical-computational gap. More precisely, we study
the problem of estimating a bivariate isotonic matrix with unknown
permutations acting on its rows and columns, given noisy, partial
observations of its entries; this matrix class strictly contains the
SST model~\citep{Cha15,ShaBalGunWai17} for ranking from pairwise
comparisons. As a corollary of our results, we show that when the
underlying matrix has dimension $n \times n$ and $\Theta(n^2)$ noisy
entries are observed, our polynomial-time, two-dimensional sorting
algorithm provably achieves the rate of estimation $\widetilde
\order(n^{-3/4})$ in the normalized Frobenius norm; thus, this result
breaks the previously mentioned $\widetilde \order(n^{-1/2})$
barrier~\citep{ShaBalGunWai17,ChaMuk16}. Although the rate $\widetilde
\order(n^{-3/4})$ still differs from the minimax optimal rate $\widetilde
\Theta(n^{-1})$, our algorithm is, to the best of our knowledge, the
first efficient procedure to obtain a rate faster than $\widetilde
\order(n^{-1/2})$ uniformly over the SST class. This guarantee, which is
stated in slightly more technical terms below, can be significant in
practice (see Figure~\ref{fig:plot}).

\paragraph{Main theorem (informal)}
{\it There is an estimator $\Mhat$ computable in time $\order(n^{2.5})$ such that for any $n \times n$ SST matrix $M^*$, given $\Theta(n^2)$ Bernoulli observations of its entries, we have}
\begin{align*}
\mathbb{E} \left[ \frac{1}{n^2} \|\Mhat - M^* \|_F^2\right] \leq C \left( \frac{\log n}{n} \right)^{3/4}.
\end{align*}
\begin{figure}[ht]
\centering
\begin{minipage}[c]{.5\linewidth}
\includegraphics[clip, trim=11cm 6.5cm 6.2cm 6.3cm, width=\linewidth]{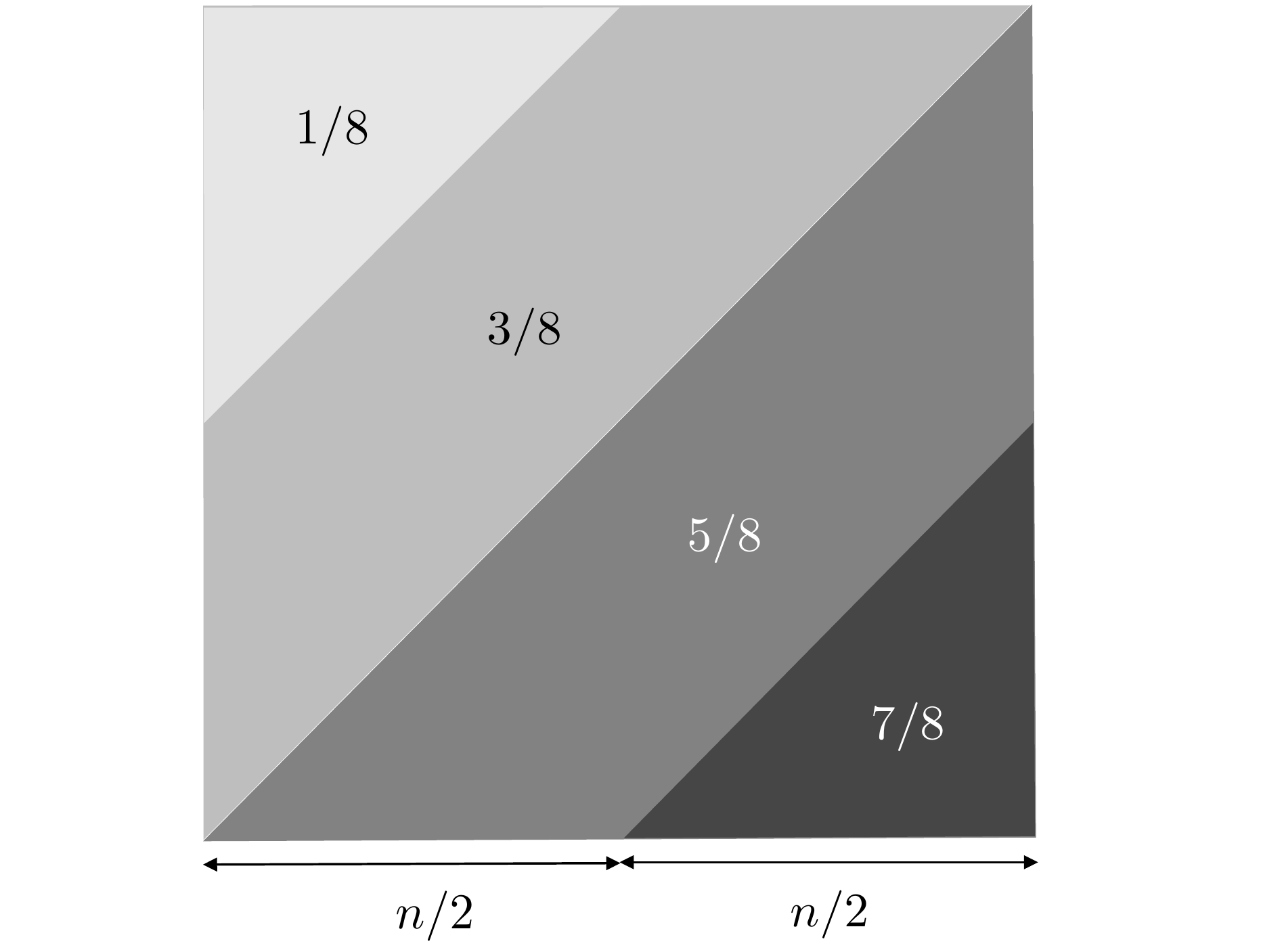}
\end{minipage}
\begin{minipage}[c]{.48\linewidth}
\includegraphics[clip, trim=1.4cm 6.45cm 2.4cm 6.7cm, width=\linewidth]{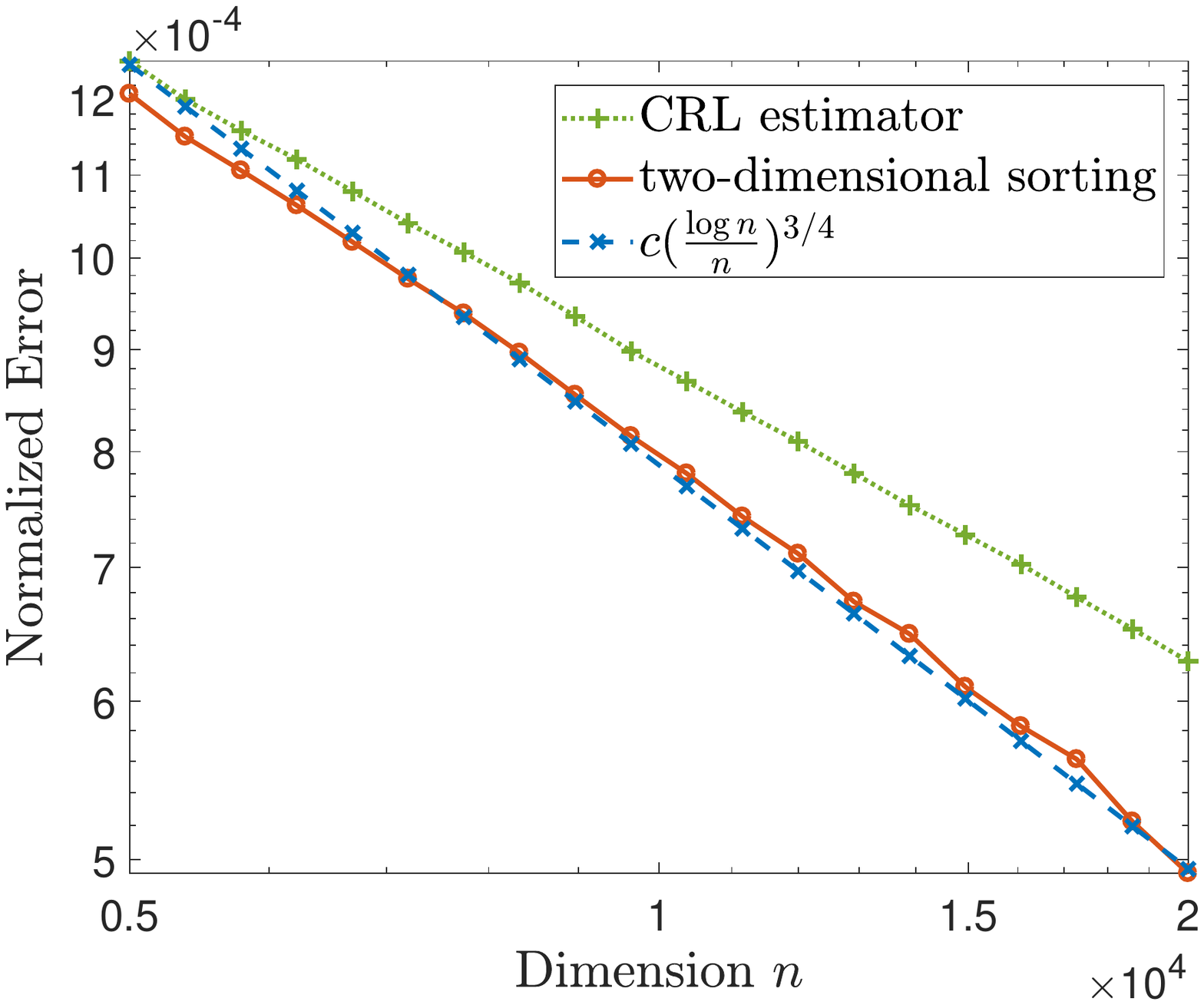}
\end{minipage} 
\caption{\textbf{Left:} A bivariate isotonic matrix; $M^* \in [0, 1]^{n \times n}$ is a row and column permuted version of such a matrix.  \textbf{Right:} A log-log plot
  of the error \mbox{$\frac{1}{n^2} \|\Mhat - M^*\|_F^2$}
  (averaged over $10$ experiments each using $n^2$ Bernoulli observations) of our estimator and the $\CRL$ estimator~\citep{ShaBalWai16-2}.} 
\label{fig:plot}
\end{figure}



Our algorithm is novel in the sense that it is neither spectral
in nature, nor simple variations of the Borda count estimator that was
previously employed. Our algorithm takes advantage of the fine
monotonicity structure of the underlying matrix along both dimensions,
and this allows us to prove tighter bounds than before. In addition to
making algorithmic contributions, we also briefly revisit the minimax
rates of estimation.

\paragraph{Organization.}

In Section~\ref{sec:setup}, we formally introduce our estimation
problem. Section~\ref{sec:mainresults} contains statements and
discussions of our main results, and in Section~\ref{sec:app}, we
describe in detail how the estimation problem that we study is
connected to applications in crowd-labeling and ranking from pairwise
comparisons. We provide the proofs of our main results in Section~\ref{sec:proofs}.

\paragraph{Notation.}
For a positive integer $n$, let $[n] \defn \{1, 2, \ldots, n\}$. For a
finite set $S$, we use $|S|$ to denote its cardinality. For two
sequences $\{a_n\}_{n=1}^\infty$ and $\{b_n\}_{n=1}^\infty$, we write
$a_n \lesssim b_n$ if there is a universal constant $C$ such that $a_n
\leq C b_n$ for all $n \geq 1$. The relation $a_n \gtrsim b_n$ is
defined analogously.  We use $c, C, c_1, c_2, \dots$ to denote
universal constants that may change from line to line.  We use
$\BER(p)$ to denote the Bernoulli distribution with success
probability $p$, the notation $\BIN(n,p)$ to denote the binomial
distribution with $n$ trials and success probability $p$, and the
notation $\Poi(\lambda)$ to denote the Poisson distribution with
parameter $\lambda$.  Given a matrix $M \in \real^{\dimone \times
  \dimtwo}$, its $i$-th row is denoted by $M_i$. For a vector $v \in \mathbb{R}^n$, define its variation as $\var(v) = \max_i v_i - \min_i v_i$. Let $\symgp_n$ denote
the set of all permutations $\pi: [n] \to [n]$. Let $\id$ denote the
identity permutation, where the dimension can be inferred from
context.


\section{Background and problem setup} \label{sec:setup}

In this section, we present the relevant technical background and
notation on permutation-based models, and introduce the observation
model of interest.

\subsection{Matrix models}

Our main focus is on designing efficient algorithms for estimating a
bivariate isotonic matrix with unknown permutations acting on its rows
and columns. Formally, we define $\Cbiso$ to be the class of matrices
in $[0,1]^{\dimone \times \dimtwo}$ with nondecreasing rows and
nondecreasing columns. For readability, we assume throughout that
$\dimone \geq \dimtwo$ unless otherwise stated; our results can be
straightforwardly extended to the other case. Given a matrix $M \in
\real^{\dimone \times \dimtwo}$ and permutations $\pi \in
\symgp_{\dimone}$ and $\sigma \in \symgp_{\dimtwo}$, we define the
matrix $M(\pi,\sigma) \in \real^{\dimone \times \dimtwo}$ by
specifying its entries as
\begin{align*}
\left[ M(\pi,\sigma)\right]_{i, j} = M_{\pi(i), \sigma(j)} \text{ for
} i \in [\dimone], j \in [\dimtwo].
\end{align*}
Also define the class $ \Cbiso(\pi, \sigma) \defn \{M(\pi, \sigma): M
\in \Cbiso\} $ as the set of matrices that are bivariate isotonic when
viewed along the row permutation $\pi$ and column permutation
$\sigma$, respectively.

The class of matrices that we are interested in estimating contains bivariate isotonic matrices whose rows and columns are both acted upon by unknown permutations:
\begin{align*}
\Cperm \defn \bigcup_{\substack{\pi \in \symgp_{\dimone} \\ \sigma \in
    \symgp_{\dimtwo}}} \Cbiso(\pi, \sigma) .
\end{align*}

\subsection{Observation model} \label{sec:obs}

In order to study estimation from noisy observations of a matrix $M^*$
in the class $\Cperm$, we suppose that $N$
noisy entries are sampled independently and uniformly with replacement
from all entries of $M^*$.  This sampling model is popular in the
matrix completion literature, and is a special case of the \emph{trace
  regression model}~\citep{NegWai10b,KolLouTsy11}. It has also been
used in the context of permutation models by Mao et
al.~\citet{MaoWeeRig17} to study the noisy sorting class.

More precisely, let $E^{(i,j)}$ denote the $\dimone \times \dimtwo$
matrix with $1$ in the $(i,j)$-th entry and $0$ elsewhere, and suppose
that $X_\ell$ is a random matrix sampled independently and uniformly
from the set $\{E^{(i,j)}: i \in [\dimone], \, j \in [\dimtwo]\}$. We
observe $N \leq \dimone \dimtwo$ independent pairs $\{(X_\ell,
y_\ell)\}_{\ell=1}^N$ from the model
\begin{align} \label{eq:model}
y_\ell = \trace(X_\ell^\top M^*) + z_\ell,
\end{align}
where the observations are contaminated by independent, centered,
sub-Gaussian noise $z_\ell$ with variance parameter $\vars^2$. 
Of particular interest is the noise model considered in
applications such as crowd-labeling and ranking from pairwise
comparisons.  Here our samples take the form
\begin{align}
y_\ell \sim \BER\big(\trace(X_\ell^\top M)\big) \label{eq:Bernoise}
\end{align}
and consequently, the sub-Gaussian parameter $\vars^2$ is bounded; for a discussion of other regimes of noise in a related matrix model, see Gao~\cite{Gao17}.

For analytical convenience, we employ the standard trick of Poissonization, whereby we assume throughout the paper that $N' = \Poi(N)$ random samples are drawn according to the trace regression model~\eqref{eq:model}. Upper and lower bounds derived under this model carry over with loss of constant factors to the model with exactly $N$ samples; for a detailed discussion, see Appendix~\ref{app:poi}.

For notational convenience, denote the probability that an entry of the matrix is observed under Poissonized sampling by $\pobs = 1 - \exp( - N / \dimone \dimtwo)$. Since we assume throughout that $N \leq \dimone \dimtwo$, it can be verified that $\frac{N}{2 \dimone \dimtwo} \leq \pobs \leq \frac{N}{\dimone \dimtwo}$.

Now given $N' = \Poi(N)$ observations $\{(X_\ell, y_\ell)\}_{\ell=1}^{N'}$, let us define the matrix of observations $Y = Y \left(\{(X_\ell, y_\ell)\}_{\ell=1}^{N'}\right)$, with entry $(i, j)$ given by
\begin{align} \label{eq:obs-Y}
Y_{i,j} = \frac{1}{\pobs} \frac{1}{1 \vee \sum_{\ell = 1}^{N'} \bfone \{ X_\ell = E^{(i,j)} \}} \sum_{\ell = 1}^{N'} y_\ell \, \bfone \{ X_\ell = E^{(i,j)} \}.
\end{align}

In words, the rescaled entry $\pobs Y_{i,j}$ is the average of all the noisy realizations of $M^*_{i,j}$ that we have observed, or zero if the entry goes unobserved. Note that
$
\EE[Y_{i,j}] = \frac{1}{\pobs} M^*_{i,j} \cdot \pobs = M^*_{i,j},
$
so that $\E[Y] = M^*$.
Moreover, we may write the model in the linearized form $Y = M^* + W$,
where $W$ is a matrix of additive noise having independent, zero-mean, sub-Gaussian entries.


\section{Main results} \label{sec:mainresults}
In this section, we present our main results---we begin by briefly revisiting the fundamental limits of estimation, and then introduce our algorithms in Section~\ref{sec:ex-algo}. We assume throughout this section that as per the setup, we have $\dimone \geq \dimtwo$ and $N \in [\dimone \dimtwo]$.
\subsection{Statistical limits of estimation}
We begin by characterizing the fundamental limits of estimation under the trace regression observation model~\eqref{eq:model} with $N' = \Poi(N)$ observations. 
We define the least squares estimator over the class of matrices~$\Cperm$ as the projection
\begin{align*}
\Mhatls(Y) \defn \arg\min_{M \in \Cperm} \|Y - M \|_F^2.
\end{align*}
The projection is a non-convex problem, and is unlikely to be computable exactly in polynomial time. However, studying this estimator allows us to establish a baseline that characterizes the best achievable statistical rate. The following theorem characterizes its risk up to a logarithmic factor in the dimension; recall the shorthand $Y = Y \left(\{X_{\ell}, y_{\ell} \}_{\ell = 1}^{N'}\right)$.
\begin{theorem} \label{thm:funlim}
For any matrix $M^* \in \Cperm$, we have
\begin{subequations}
\begin{align} \label{eq:ls-upper}
\frac{1}{\dimone \dimtwo} \| \Mhatls(Y) - M^* \|_F^2 \lesssim \maxvarone \frac{\dimone \log^2 \dimone}{N}
\end{align}
with probability at least $1 - (\dimone \dimtwo)^{-3}$. 

Additionally, under the Bernoulli observation model~\eqref{eq:Bernoise}, any estimator $\Mhat$ satisfies 
\begin{align} \label{eq:lower}
\sup_{M^* \in \Cperm} \EE \left[ \frac{1}{\dimone \dimtwo} \| \Mhat - M^* \|_F^2 \right] \gtrsim \frac{\dimone}{N}. 
\end{align}
\end{subequations}
\end{theorem}

The factor $\maxvarone$ appears in the upper bound instead of the noise variance $\vars^2$ because even if the noise is zero, there are missing entries.
The theorem characterizes the minimax rate of estimation for the class $\Cperm$ up to a logarithmic factor. 
\subsection{Efficient algorithms} 
\label{sec:ex-algo}


Next, we propose polynomial-time algorithms for estimating the permutations $(\pi, \sigma)$ and the matrix $M^*$. 
Our main algorithm relies on two distinct steps: first, we estimate the unknown permutations, and then project onto the class of matrices that are bivariate isotonic when viewed along the estimated permutations. The formal meta-algorithm is described below. 


\paragraph{Algorithm 1 (meta-algorithm)}
\begin{itemize}
\item Step 0: Split the observations into two disjoint parts, each containing $N'/2$ observations, and construct the matrices $Y^{(1)} = Y \left(\{X_\ell, y_\ell \}_{\ell = 1}^{N'/2} \right)$ and $Y^{(2)} = Y \left(\{X_\ell, y_\ell \}_{\ell = N'/2 + 1}^{N'} \right)$.
\item Step 1: Use $Y^{(1)}$ to obtain the permutation estimates $(\pihat, \sigmahat)$.
\item Step 2: Return the matrix estimate 
$
\Mhat(\pihat, \sigmahat) \defn \arg \min_{M \in \Cbiso(\pihat, \sigmahat)} \| Y^{(2)} - M \|_F^2 .
$
\end{itemize}
Owing to the convexity of the set $\Cbiso(\pihat, \sigmahat)$, the projection operation in Step 2 of the algorithm can be computed in near linear time~\citep{BriDykPilRob84, KynRaoSac15}. 
The following result, a slight variant of Proposition~4.2 of Chatterjee and Mukherjee~\cite{ChaMuk16}, allows us to characterize the error rate of any such meta-algorithm as a function of the permutation estimates $(\pihat, \sigmahat)$.

\begin{proposition}
\label{prop:meta}
Suppose that $M^* \in \Cbiso(\pi, \sigma)$ where $\pi$ and $\sigma$
are unknown permutations in $\symgp_{\dimone}$ and $\symgp_{\dimtwo}$
respectively.  Then with probability at least $1 - (\dimone
\dimtwo)^{-3}$, we have
\begin{align} 
\frac{1}{\dimone \dimtwo} \big\| \Mhat(\pihat, \sigmahat) - M^*
\big\|_F^2 \lesssim \maxvarone \frac{\dimone \log^2 \dimone}{N} & +
\frac{1}{\dimone \dimtwo} \big\|M^*(\pi^{-1} \circ \pihat, \id) - M^*
\big\|_F^2 \notag \\ & + \frac{1}{\dimone \dimtwo} \big\| M^*(\id,
\sigma^{-1} \circ \sigmahat) - M^* \big\|_F^2. \label{eq:oracle}
\end{align}
\end{proposition}

The first term on the right hand side of the bound~\eqref{eq:oracle}
corresponds to an estimation error, if the true permutations $\pi$ and
$\sigma$ were known a priori, and the latter two terms correspond to
an approximation error that we incur as a result of having to estimate
these permutations from data.  Comparing the bound~\eqref{eq:oracle}
to the minimax lower bound~\eqref{eq:lower}, we see that up to a
logarithmic factor, the first term of the bound~\eqref{eq:oracle} is
unavoidable, and so we can restrict our attention to obtaining good
permutation estimates $(\pihat, \sigmahat)$.  
We now present our main permutation estimation procedure that can be plugged into Step~1 of this meta-algorithm.

\subsubsection{Two-dimensional sorting} \label{sec:tds}

%

Since the underlying matrix of interest is individually monotonic along each  dimension, the row and column sums provide noisy information about the respective unknown permutations. Consequently, variants of such procedures are popular in the literature~\cite{ChaMuk16,FlaMaoRig16}. However, such a procedure does not take simultaneous advantage of the fact that the underlying matrix is monotonic in \emph{both} dimensions. To improve upon simply sorting row 
(resp. column) 
sums, we propose an algorithm that first sorts the columns 
(resp. rows) 
of the matrix approximately, and then exploits this approximate ordering to sort the rows 
(resp. columns) 
of the matrix.

We need more notation to facilitate the description of the
algorithm. For a partition \\ \mbox{$\bl = (\bl_1, \dots, \bl_K)$ of
  the set $[\dimtwo]$}\footnote{$\bl$ is a partition of $[\dimtwo]$ if
  $[\dimtwo] = \cup_{k=1}^K \bl_k$ and $\bl_j \cap \bl_k = \emptyset$
  for $j \neq k$}, we group the columns of a matrix $Y \in
\real^{\dimone \times \dimtwo}$ into $K$ blocks according to their
indices in $\bl$, and refer to $\bl$ as a partition or \emph{blocking}
of the columns of $Y$.

Given a data matrix $Y \in \real^{\dimone \times \dimtwo}$, the
following blocking subroutine returns a column partition $\BL(Y)$. In
the main algorithm, partial row sums are computed on indices contained
in each block.


\paragraph{Subroutine 1 (blocking)}
\begin{itemize}
\item Step 1: Compute the column sums $\{C(j)\}_{j = 1}^{\dimtwo}$ of
  the matrix $Y$ as
\begin{align*}
C(j) = \sum_{i=1}^{\dimone} Y_{i,j}.
\end{align*} 
Let $\sigmahatpre$ be the permutation along which the sequence
$\{C(\sigmahatpre(j))\}_{j=1}^{\dimtwo}$ is nondecreasing.

\item Step 2: Set $\tau = 16 \varplusone \Big( \sqrt{\frac{\dimone^2
    \dimtwo}{N} \log(\dimone \dimtwo) } + \frac{\dimone \dimtwo}{N}
  \log(\dimone \dimtwo) \Big)$ and $K = \lceil \dimtwo/\tau
  \rceil$. Partition the columns of $Y$ into $K$ blocks by defining
\begin{align*}
\bl_1 &= \{j \in [\dimtwo]: C(j) \in (-\infty, \tau) \}, \\ \bl_k &=
\left\{j \in [\dimtwo] : C(j) \in \big[ (k - 1) \tau , k \tau \big)
  \right\} \text{ for } 1<k<K, \text{ and} \\ \bl_K &= \{j \in
           [\dimtwo]: C(j) \in [(K-1)\tau, \infty)\}.
\end{align*}
Note that each block is contiguous when the columns are permuted by
$\sigmahatpre$.

\item Step 3 (aggregation): Set $\beta = \blocksize$. Call a block
  $\bl_k$ ``large'' if $|\bl_k| \geq \beta$ and ``small"
  otherwise. Aggregate small blocks in $\bl$ while leaving the large
  blocks as they are, to obtain the final partition $\BL$.

More precisely, consider the matrix $Y' = Y(\id, \sigmahatpre)$ having
nondecreasing column sums and contiguous blocks. Call two small blocks
``adjacent'' if there is no other small block between them.  Take
unions of adjacent small blocks to ensure that the size of each
resulting block is in the range $[ \frac{1}{2} \beta, 2 \beta]$. If
the union of all small blocks is smaller than $\frac{1}{2} \beta$,
aggregate them all.

Return the resulting partition $\BL(Y) = \BL$.
\end{itemize}

The threshold $\tau$ is a chosen to be a high probability bound on the perturbation of any column sum. In particular, this ensures that we obtain blocks containing columns that are close when the matrix is ordered according to the correct permutation. Computing partial row sums within each block then provides more refined information about the underlying row permutation than simply computing full row sums, and this intuition underlies the two-dimensional sorting algorithm to follow.
%
As a technical detail, it is important to note that Step~3 aggregates small blocks into large enough ones to reduce noise in these partial row sums. We are now in
a position to describe the two-dimensional sorting algorithm.


\paragraph{Algorithm 2 (two-dimensional sorting)}
\begin{itemize}
\item Step 0: Split the observations into two independent subsamples
  of equal size, and form the corresponding matrices $Y^{(1)}$ and
  $Y^{(2)}$ according to equation~\eqref{eq:obs-Y}.

\item Step 1: Apply Subroutine 1 to the matrix $Y^{(1)}$ to obtain a
  partition $\BL = \BL(Y^{(1)})$ of the columns. Let $K$ be the number
  of blocks in $\BL$.

\item Step 2: Using the second sample $Y^{(2)}$, compute the row sums
\begin{align*}
S(i) = \sum_{j \in [\dimtwo]} Y^{(2)}_{i,j} \text{ for each }i \in
[\dimone],
\end{align*} 
and the partial row sums within each block 
\begin{align*}
S_{\BL_k}(i) = \sum_{j \in \BL_k} Y^{(2)}_{i,j} \text{ for each }i \in
[\dimone], k \in [K].
\end{align*}
Create a directed graph $G$ with vertex set $[\dimone]$, where an edge
$u \to v$ is present if either
\begin{subequations}
\begin{align}
S(v) - S(u) & > 16 \varplusone \bigg( \sqrt{\frac{\dimone
    \dimtwo^2}{N} \log(\dimone \dimtwo) } + \frac{\dimone \dimtwo}{N}
\log(\dimone \dimtwo) \bigg), \text{ or} \label{eq:full-sum}
\\ S_{\BL_k}(v) - S_{\BL_k}(u) & > 16 \varplusone \bigg(
\sqrt{\frac{\dimone \dimtwo}{N} |\BL_k| \log(\dimone \dimtwo) } +
\frac{\dimone \dimtwo}{N} \log(\dimone \dimtwo) \bigg) \text{ for some
} k \in [K]. \label{eq:block-sum}
\end{align}
\end{subequations}

\item Step 3: Compute a topological sort $\pihattds$ of the graph $G$;
  if none exists, set $\pihattds = \id$.

\item Step 4: Repeat Steps 1--3 with $(Y^{(i)})^\top$ replacing
  $Y^{(i)}$ for $i=1,2$, the roles of $\dimone$ and $\dimtwo$
  switched, and the roles of $\pi$ and $\sigma$ switched, to compute
  the permutation estimate $\sigmahattds$.

\item Step 5: Return the permutation estimates $(\pihattds,
  \sigmahatftds)$.
\end{itemize}

Recall that a permutation $\pi$ is called a topological sort of $G$ if $\pi(u)<\pi(v)$ for every directed edge $u \to v$.
The construction of the graph $G$ in Step~2 dominates the
computational complexity, and takes time $\order(\dimone^2 \dimtwo /
\beta) = \order(\dimone^2 \dimtwo^{1/2})$. We have the following
guarantee for the two-dimensional sorting algorithm.

\begin{theorem} \label{thm:fast-tds}
For any matrix $M^* \in \Cperm$, we have
\begin{align*}
\frac{1}{\dimone \dimtwo} \big\|\Mhat(\pihattds, \sigmahatftds) - M^*
\big\|_F^2 \lesssim \maxvarone \left[ \Big(\frac{\dimone \log
    \dimone}{N} \Big)^{3/4} + \frac{\dimone \log^2 \dimone}{N} \right]
\end{align*}
with probability at least $1- 9(\dimone \dimtwo)^{-3}$.
\end{theorem}

In particular, setting $N = \dimone \dimtwo$, we have proved that our
efficient estimator enjoys the rate
\begin{align*}
\frac{1}{\dimone \dimtwo} \big\|\Mhat(\pihattds, \sigmahatftds) - M^*
\big\|_F^2 = \widetilde O \left(\dimtwo^{-3/4}\right),
\end{align*}
which is the main theoretical guarantee established in this paper for
permutation-based models.


\section{Applications}
\label{sec:app}

We now discuss in detail how the matrix models studied in this paper
arise in practice.  The class $\Cperm$ was studied as a
permutation-based model for crowd-labeling~\citep{ShaBalWai16} in the
case of binary questions, and was proposed as a strict generalization
of the classical Dawid-Skene model~\citep{DawSke79, KarOhSha11-b,
  LiuPenIhl12, DasDasRas13, GhoKalMcA11}. Here there is a set of
$\dimtwo$ questions of a binary nature; the true answer to these
questions can be represented by a vector $x^* \in \{0, 1\}^{\dimtwo}$,
and our goal is to estimate this vector by asking these questions to
$\dimone$ \emph{workers} on a crowdsourcing platform. A key to this
problem is being able to model the probabilities with which workers
answer questions correctly, and we do so by collecting these
probabilities within a matrix $M^* \in [0,1]^{\dimone \times
  \dimtwo}$. Assuming that workers have a strict ordering $\pi$ of
their abilities, and that questions have a strict ordering $\sigma$ of
their difficulties, the matrix $M^*$ is bivariate isotonic when the
rows are ordered in increasing order of worker ability, and columns
are ordered in decreasing order of question difficulty. However, since
worker abilities and question difficulties are unknown a priori, the
matrix of probabilities obeys the inclusion $M^* \in \Cperm$.

In the \emph{calibration} problem, we would like to ask questions
whose answers we know a priori, so that we can estimate worker
abilities and question difficulties, or more generally, the entries of
the matrix $M^*$. This corresponds to estimating matrices in the class
$\Cperm$ from noisy observations of their entries, whose rate of estimation is our main result.  

A subclass of $\Cperm$ specializes to the case $\dimone = \dimtwo =
n$, and also imposes an additional skew symmetry constraint. More
precisely, define $\Cbiso'$ analogously to the class $\Cbiso$, except
with matrices having columns that are nonincreasing instead of
nondecreasing. Also define the class $\Cskew(n) \defn \{M \in [0,
  1]^{\dimone \times \dimtwo}: M + M^\top = 11^\top \}$, and the
\emph{strong stochastic transitivity} class
\begin{align*}
\csst(n) \defn \left( \bigcup_{\pi \in \symgp_{n}} \Cbiso'(\pi, \pi)
\right) \bigcap \Cskew(n).
\end{align*}

The class $\csst(n)$ is useful as a model for estimation from pairwise
comparisons~\citep{Cha15, ShaBalGunWai17}, and was proposed as a
strict generalization of parametric models for this
problem~\citep{BraTer52, NegOhSha16, RajAga14}. In particular, given
$n$ items obeying some unknown underlying ranking $\pi$, entry $(i,
j)$ of a matrix $M^* \in \csst(n)$ represents the probability $\Pr(i
\succ j)$ with which item $i$ beats item $j$ in a pairwise comparison
between them. The shape constraint encodes the transitivity condition
that for all triples $(i, j, k)$ obeying $\pi(i) < \pi(j) < \pi(k)$,
we must have
\begin{align*}
\Pr(i \succ k) \geq \max\{\Pr(i \succ j), \Pr(j \succ k)\}.
\end{align*}
For a more classical introduction to these models, see the
papers~\citet{Fis73, McLLuc65, BalWil97} and the references
therein. Our task is to estimate the underlying ranking from results
of passively chosen pairwise comparisons\footnote{Such a passive,
  simultaneous setting should be contrasted with the \emph{active}
  case (e.g.,~\cite{HecShaRamWai16, FalOrlPicSur17, AgaAgaAssKha17}),
  where we may sequentially choose pairs of items to compare depending
  on the results of previous comparisons.}  between the $n$ items, or
more generally, to estimate the underlying probabilities $M^*$ that
govern these comparisons\footnote{Accurate, proper estimates of $M^*$
  translate to accurate estimates of the ranking $\pi$ (see Shah et
  al.~\cite{ShaBalGunWai17}).}. All the results we obtain in this work
clearly extend to the class $\csst(n)$ with minimal modifications; for
example, either of the two estimates $\pihattds$ or $\sigmahattds$ may
be returned as an estimate of the permutation $\pi$. Consequently, the
informal theorem stated in the introduction is an immediate corollary
of Theorem~\ref{thm:fast-tds} once these modifications are made to the
algorithm.


\section{Proofs}
\label{sec:proofs}

Throughout the proofs, we assume without loss of generality that $M^*
\in \Cbiso(\id,\id) = \Cbiso$.  Because we are interested in rates of
estimation up to universal constants, we assume that each independent
subsample contains $N' = \Poi(N)$ observations (instead of $\Poi(N)/2$
or $\Poi(N)/4$). We use the shorthand $Y = Y \left(\{(X_\ell,
y_\ell)\}_{\ell=1}^{N'}\right)$, throughout.


\subsection{Some preliminary lemmas}

Before turning to the proof of Theorems
\ref{thm:funlim} and~\ref{thm:fast-tds}, we
provide three lemmas that underlie many of our arguments.  The first
lemma can be readily distilled from the proof of Theorem~5 of Shah et
al.~\cite{ShaBalGunWai17} with slight modifications.  It is worth
mentioning that similar lemmas characterizing the estimation error of
a bivariate isotonic matrix were also proved
by~\citet{ChaGunSen18,ChaMuk16}.

\begin{lemma}[\cite{ShaBalGunWai17}]
  \label{lem:shah}
Let $\dimone \geq \dimtwo$, and let $M^* \in \Cperm$. Assume that our
observation model takes the form $Y = M^* + W$, where the noise matrix
$W$ satisfies the properties
\begin{enumerate}
\item[(a)] the entries $W_{i,j}$ are independent, centered,
  $\frac{c_1}{\pobs} (\vars \lor 1)$-sub-Gaussian random variables;
\item[(b)] the second moments are bounded as $\E[|W_{i,j}|^2] \le
  \frac{c_2}{\pobs} \maxvarone$ for all $i \in [\dimone], j \in
       [\dimtwo]$.
\end{enumerate}
Then the least squares estimator $\Mhatls(Y)$ satisfies
\begin{align*}
\Pr \left\{ \Big\| \Mhatls(Y) - M^* \Big\|_F^2 \ge
\frac{c_3}{\pobs} \maxvarone \dimone \log^2 \dimone \right\} \le
(\dimone \dimtwo)^{-3}.
\end{align*}
Moreover, the same result holds if the class $\Cperm$ is replaced by
the class $\Cbiso$.
\end{lemma}

The proof closely follows that of Shah et
al.~\citet[Theorem~5]{ShaBalGunWai17}; consequently, we postpone it to Appendix~\ref{app:ls}.
The next lemma establishes concentration of sums of our observations
around their means.
\begin{lemma}
  \label{lem:par-sum}
For any nonempty subset $\mathcal{S} \subset [\dimone] \times
[\dimtwo]$, it holds that
\begin{align*}
\Pr \left\{ \bigg| \sum_{(i,j) \in \mathcal{S}} (Y_{i,j} - M^*_{i,j})
\bigg| \ge 8 \varplusone \bigg( \sqrt{\frac{|\mathcal{S}| \dimone
    \dimtwo}{N} \log(\dimone \dimtwo) } + 2 \frac{\dimone \dimtwo}{N}
\log(\dimone \dimtwo) \bigg) \right\} \le 2 (\dimone \dimtwo)^{-4} .
\end{align*}
\end{lemma}

\begin{proof}
According to definitions \eqref{eq:model} and \eqref{eq:obs-Y}, we
have
\begin{align*}
W_{i,j} = Y_{i,j} - M^*_{i,j} =
\begin{cases}
- M^*_{i,j} \text{ if entry $(i,j)$ is not observed, and}
\\ M^*_{i,j}/\pobs - M^*_{i,j} + \frac{W'_{i,j}}{\pobs}, \text{
  otherwise,}
\end{cases}
\end{align*}
where $W'$ is a $\vars$-sub-Gaussian noise matrix with independent
entries.  Consequently, we can express the noise on each entry as
$W_{i,j} = Z^{(1)}_{i,j} + Z^{(2)}_{i,j}$ where $\{ Z^{(1)}_{i,j}
\}_{i \in [\dimone], j \in [\dimtwo]}$ are independent, zero-mean
random variables given by
\begin{align*}
Z^{(1)}_{i,j} = 
\begin{cases}
M^*_{i,j}(\pobs^{-1} - 1) & \text{ with probability } \pobs,
\\ -M^*_{i,j} & \text{ with probability } 1-\pobs,
\end{cases}
\end{align*}
and $\{ Z^{(2)}_{i,j} \}_{i \in [\dimone], j \in [\dimtwo]}$ are
independent, zero-mean random variables such that
\begin{align*}
Z^{(2)}_{i,j} \text{ is }
\begin{cases}
\frac{\vars}{\pobs}\text{-sub-Gaussian} & \text{ with probability }
\pobs, \\ 0 & \text{ with probability } 1-\pobs.
\end{cases}
\end{align*}

We control the two separately. First, we have $|Z^{(1)}_{i,j}| \le
1/\pobs$ and the variance of each $Z^{(1)}_{i,j}$ is bounded by
$(1-\pobs)^2/\pobs + (1-\pobs) \le 1/\pobs$. Hence Bernstein's
inequality for bounded noise yields
\begin{align*}
\Pr \left\{ \bigg| \sum_{(i,j) \in \mathcal{S}} Z^{(1)}_{i,j} \bigg|
\ge t \right\} \le 2 \exp\Big(- \frac{t^2/2}{|\mathcal{S}|/\pobs +
  t/(3\pobs)} \Big).
\end{align*}
Taking $t = 4 \sqrt{\frac{|\mathcal{S}| \dimone \dimtwo}{N}
  \log(\dimone \dimtwo) } + 6 \frac{\dimone \dimtwo}{N} \log(\dimone
\dimtwo)$ and recalling that $\pobs \ge \frac{N}{2 \dimone \dimtwo}$,
we obtain
\begin{align*}
\Pr \left\{ \bigg| \sum_{(i,j) \in \mathcal{S}} Z^{(1)}_{i,j} \bigg|
\ge 4 \sqrt{\frac{|\mathcal{S}| \dimone \dimtwo}{N} \log(\dimone
  \dimtwo) } + 6 \frac{\dimone \dimtwo}{N} \log(\dimone \dimtwo)
\right\} \le (\dimone \dimtwo)^{-4} .
\end{align*}

In order to control the deviation of the sum of $Z^{(2)}_{i,j}$, we
note that the $q$-th moment of $Z^{(2)}_{i,j}$ is bounded by
$\frac{N}{\dimone \dimtwo} (\frac{2 \vars}{\pobs} \sqrt{q})^q \le
\frac{q!}{2} \frac{8 \vars^2 \dimone \dimtwo}{N} (\frac{4 \vars
  \dimone \dimtwo}{N} )^{q-2}.$ Then another version of Bernstein's
inequality~\citep{BouLugMas13} yields
\begin{align*}
\Pr \left\{ \bigg| \sum_{(i,j) \in \mathcal{S}} Z^{(2)}_{i,j} \bigg|
\ge \sqrt{\frac{16 \vars^2 |\mathcal{S}| \dimone \dimtwo}{N} t} +
\frac{4 \vars \dimone \dimtwo}{N} t \right\} \le 2 \exp(-t),
\end{align*}
and setting $t = 4 \log(\dimone \dimtwo)$ gives
\begin{align*}
\Pr \left\{ \bigg| \sum_{(i,j) \in \mathcal{S}} Z^{(2)}_{i,j} \bigg|
\ge 8 \vars \sqrt{\frac{|\mathcal{S}| \dimone \dimtwo}{N} \log(\dimone
  \dimtwo) } + 16 \vars \frac{\dimone \dimtwo}{N} \log(\dimone
\dimtwo) \right\} \le (\dimone \dimtwo)^{-4} .
\end{align*}
Combining the above two deviation bounds completes the proof.
\end{proof}

\noindent The last lemma is a deterministic result.
\begin{lemma}
  \label{lem:per-num}
Let $\{a_i\}_{i=1}^{n}$ be a nondecreasing sequence of real
numbers. If $\pi$ is a permutation in $\symgp_{n}$ such that $\pi(i) <
\pi(j)$ whenever $a_j - a_i > \tau$ where $\tau > 0$, then
$|a_{\pi(i)} - a_i| \le \tau$ for all $i \in [n]$.
\end{lemma}

\begin{proof}
Suppose that $a_j - a_{\pi(j)} > \tau$ for some index $j \in
[n]$. Since $\pi$ is a bijection, there must exist an index $i \le
\pi(j)$ such that $\pi(i) > \pi(j)$. However, we then have $a_j - a_i
\ge a_j - a_{\pi(j)} > \tau$, which contradicts the assumption. A
similar argument shows that $a_{\pi(j)} - a_j > \tau$ also leads to a
contradiction. Therefore, we obtain that $|a_{\pi(j)} - a_j| \le \tau$
for every $j \in [n]$.
\end{proof}

\noindent With these lemmas in hand, we are now ready to prove our
main theorems.


\subsection{Proof of Theorem~\ref{thm:funlim}}

We split the proof into two parts by proving the upper and lower
bounds separately.

\subsubsection{Proof of upper bound}

The upper bound follows from Lemma~\ref{lem:shah} once we check the conditions on the noise for our model. We have seen in the proof of Lemma~\ref{lem:par-sum} that the noise on each entry can be written as $W_{i,j} = Z^{(1)}_{i,j} + Z^{(2)}_{i,j}$. Again, $Z^{(1)}_{i,j}$ and $Z^{(2)}_{i,j}$ are $\frac{c}{\pobs}$-sub-Gaussian and $\frac{c\, \vars}{\pobs}$-sub-Gaussian respectively, and have variances bounded by $\frac{1}{\pobs}$ and $\frac{c\, \vars^2}{\pobs}$ respectively. Hence the conditions on $W$ in Lemma~\ref{lem:shah} are satisfied. Then we can apply the lemma, recall the relation $\pobs \ge \frac{N}{2 \dimone \dimtwo}$ and normalize the bound by $\frac{1}{\dimone \dimtwo}$ to complete the proof.

\subsubsection{Proof of lower bound}

The lower bound follows from an application of Fano's lemma. The technique is standard, and we briefly review it here.
Suppose we wish to estimate a parameter
$\theta$ over an indexed class of distributions $\mathcal{P} =
\{\mathbb{P}_\theta \, \mid\, \theta \in \Theta \}$ in the square of a
(pseudo-)metric $\rho$. We refer to a subset of parameters
$\{\theta^1, \theta^2, \ldots, \theta^K \}$ as a local $(\delta,
\epsilon)$-packing set if
\begin{align*}
\min_{i,j \in [K], \, i\neq j} \rho(\theta^i, \theta^j) \geq \delta
\qquad \text{ and } \qquad \frac{1}{K(K-1)} \sum_{i, j \in [K], \, i\neq j}
D(\mathbb{P}_{\theta^i} \| \mathbb{P}_{\theta^j}) \leq \epsilon.
\end{align*}
Note that this set is a $\delta$-packing in the metric $\rho$ with the
average KL-divergence bounded by $\epsilon$.  The following result is
a straightforward consequence of Fano's inequality:
\begin{lemma}[Local packing Fano lower bound]
  \label{fanolbprime}
For any $(\delta, \epsilon)$-packing set of cardinality $K$, we have
\begin{align}
\inf_{\thetahat} \sup_{\thetastar \in \Theta} \EE \left[
  \rho(\thetahat, \thetastar)^2\right] \geq \frac{\delta^2}{2} \left(1
- \frac{\epsilon + \log 2}{\log K} \right). \label{eq:fano}
\end{align}
\end{lemma}

In addition, the Gilbert-Varshamov bound~\citep{Gil52, Var57} guarantees the
existence of binary vectors $\{v^1, v^2, \ldots, v^K \} \subseteq \{0, 1\}^{\dimone}$
such that 
\begin{subequations}
\begin{align}
K &\geq 2^{c_1 \dimone} \text{ and} \\
\| v^i - v^j \|_2^2 &\geq c_2 \dimone \text{ for each } i \neq j,
\end{align}
\end{subequations}
for some fixed tuple of constants $(c_1, c_2)$. We use this guarantee to design a packing of matrices in the class $\Cperm$. For each $i \in [K]$, fix some $\delta \in [0, 1/4]$ to be precisely set later, and define the matrix $M^i$ having identical columns, with entries given by
\begin{align}
M^i_{j, k} = 
\begin{cases}
1/2, \text{ if }v^i_j = 0 \\
1/2 + \delta, \text{ otherwise.}
\end{cases}
\end{align}
Clearly, each of these matrices $\{M^i\}_{i = 1}^{K}$ is a member of the class $\Cperm$, and each distinct pair of matrices $(M^i, M^j)$ satisfies the inequality $\| M^i - M^j \|_F^2 \geq c_2 \dimone \dimtwo \delta^2$. 

Let $\mathbb{P}_{M}$ denote the probability distribution of the observations in the model \eqref{eq:Bernoise} with underlying matrix $M \in \Cperm$. Our observations are independent across entries of the matrix, and so the KL divergence tensorizes to yield
\begin{align*}
D(\mathbb{P}_{M^i} \| \mathbb{P}_{M^j}) = \sum_{\substack{k \in [\dimone] \\ \ell \in [\dimtwo]}} D(\mathbb{P}_{M^i_{k, \ell}} \| \mathbb{P}_{M^j_{k, \ell}}).
\end{align*}
Let us now examine one term of this sum. We observe $T_{k, \ell} = \Poi(\frac{N}{\dimone \dimtwo})$ samples of entry $(k, \ell)$; conditioned on the event $T_{k,\ell} = m$, we have the distributions
\begin{align*}
\mathbb{P}_{M^i_{k,\ell}} = \BIN(m, M^i_{k, \ell}), \quad \text{and} \quad
\mathbb{P}_{M^j_{k,\ell}} = \BIN(m, M^j_{k, \ell}).
\end{align*}
Consequently, the KL divergence conditioned on $T_{k, \ell} = m$ is given by
\begin{align*}
D(\mathbb{P}_{M^i_{k, \ell}} \| \mathbb{P}_{M^j_{k, \ell}}) = m D(M^i_{k, \ell} \| M^j_{k, \ell}),
\end{align*}
where we have used $D(p \| q) = p \log ( \frac{p}{q} ) + (1 - p) \log ( \frac{1 - p}{1 - q} )$ to denote the KL divergence between the Bernoulli random variables $\BER(p)$ and $\BER(q)$.

Note that for $p, q \in [1/2, 3/4]$, we have
\begin{align*}
D(p \| q) &= p \log \left( \frac{p}{q} \right) + (1 - p) \log \left( \frac{1 - p}{1 - q} \right) \\
&\stackrel{\1}{\leq} p \left( \frac{p - q}{q} \right) + (1 - p) \left( \frac{q - p}{1 - q} \right) \\
&= \frac{(p - q)^2}{q (1 - q)} \\
&\stackrel{\2}{\leq} \frac{16}{3} (p - q)^2. 
\end{align*}
Here, step $\1$ follows from the inequality $\log x \leq x - 1$, and step $\2$ from the assumption $q \in [\frac 12, \frac 34]$.
Taking the expectation with respect to $T_{k,\ell}$, we have
\begin{align*}
D(\mathbb{P}_{M^i_{k, \ell}} \| \mathbb{P}_{M^j_{k, \ell}}) \le \frac{16}{3} \frac{N}{\dimone \dimtwo} (M^i_{k, \ell} - M^j_{k, \ell})^2 \le \frac{16}{3} \frac{N}{\dimone \dimtwo} \delta^2 ,
\end{align*}
Summing over $k \in [\dimone], \ell \in [\dimtwo]$ yields
$D(\mathbb{P}_{M^i} \| \mathbb{P}_{M^j}) \leq \frac{16}{3} N \delta^2.$

Substituting into the Fano's inequality~\eqref{eq:fano}, we have
\begin{align*}
\inf_{\Mhat} \sup_{\Mstar \in \Cperm} \EE \left[
  \| \Mhat - \Mstar\|_F^2\right] \geq \frac{c_2 \dimone \dimtwo \delta^2}{2} \left(1
- \frac{ \frac{16}{3} N \delta^2 + \log 2}{c_3 \dimone} \right).
\end{align*}
Finally, choosing $\delta^2 = c \frac{\dimone}{N}$ and normalizing by
$\dimone \dimtwo$ yields the claim.


\subsection{Proof of Proposition~\ref{prop:meta}}

Recall the definition of $\Mhat(\pihat, \sigmahat)$ in the
meta-algorithm, and additionally, define the projection of any matrix
$M \in \mathbb{R}^{\dimone \times \dimtwo}$, as
\begin{align*}
\mathcal{P}_{\pi, \sigma}(M) = \arg \min_{\Mtilde \in \Cbiso(\pi,
  \sigma)} \| M - \Mtilde \|_F^2.
\end{align*}
 and letting $W = Y^{(2)} - M^*$, we have
\begin{align}
\| \Mhat(\pihat, \sigmahat) - M^* \|_F^2 & \stackrel{\1}{\leq} 2 \|
\mathcal{P}_{\pihat, \sigmahat} (M^* + W) - \mathcal{P}_{\pihat,
  \sigmahat} (M^*(\pihat, \sigmahat) + W) \|_F^2 \notag \\
& \qquad \quad \ + 2 \| \mathcal{P}_{\pihat, \sigmahat} (M^*(\pihat,
\sigmahat) + W) - M^* \|_F^2 \notag \\
& \stackrel{\2}{\leq} 2 \| M^*(\pihat, \sigmahat) - M^* \|_F^2 + 2 \|
\mathcal{P}_{\pihat, \sigmahat} (M^*(\pihat, \sigmahat) + W) - M^*
\|_F^2 \notag \\
&\stackrel{\3}{\leq} 4 \| \mathcal{P}_{\pihat, \sigmahat} (M^*(\pihat,
\sigmahat) + W) - M^*(\pihat, \sigmahat) \|_F^2 + 6 \| M^*(\pihat,
\sigmahat) - M^* \|_F^2, \label{eq:estapprox}
\end{align}
where step $\2$ follows from the non-expansiveness of a projection onto a convex set, and steps $\1$ and $\3$ from the triangle inequality.


The first term in \eqref{eq:estapprox} is the estimation error of a bivariate isotonic matrix with known permutations. Since the sample used to obtain $(\pihat, \sigmahat)$ is independent from the sample used in the projection step, it is equivalent to control the error $\| \mathcal{P}_{\id, \id} (M^* + W) - M^* \|_F^2$. As before, the noise matrix $W$ satisfies the conditions of Lemma~\ref{lem:shah}. Therefore, applying Lemma~\ref{lem:shah} in the case $M^* \in \Cbiso$ with $\pobs \ge \frac{N}{2 \dimone \dimtwo}$ yields the desired bound of order $\maxvarone \frac{\dimone \log^2 \dimone}{N}$.


It remains to bound the second term of \eqref{eq:estapprox}, the approximation error of the permutation estimates. Note that the approximation error can be split into two components: one along the rows of the matrix, and the other along the columns. 
More explicitly, we have
\begin{align*}
\| M^* - M^*(\pihat, \sigmahat)\|_F^2 
&\leq 2\| M^* - M^*(\pihat, \id)\|_F^2 + 2\| M^*(\pihat, \id) - M^*(\pihat, \sigmahat)\|_F^2 \\
&= 2 \| M^* - M^*(\pihat, \id) \|_F^2 + 2 \| M^* - M^* (\id, \sigmahat) \|_F^2.
\end{align*}
Recall that we assumed without loss of generality that the true
permutations are identity permutations, so this completes the proof of
Proposition~\ref{prop:meta}. The proof readily extends to the general
case by precomposing $\pihat$ and $\sigmahat$ with $\pi^{-1}$ and
$\sigma^{-1}$ respectively.

\subsection{Proof of Theorem~\ref{thm:fast-tds}}

Recall that according to Proposition~\ref{prop:meta}, it suffices to
bound the approximation error of our permutation estimate $\| M^* -
M^*(\pihattds, \id) \|_F^2$.  To ease the notation, we use the shorthand
\begin{align*}
\eta \defn 16 \varplusone \Big( \sqrt{\frac{ \dimone \dimtwo^2}{N} \log(\dimone \dimtwo) } + 2 \frac{\dimone \dimtwo}{N} \log(\dimone \dimtwo) \Big) ,
\end{align*}
and for each block $\BL_k$ in Algorithm~2 where $k \in [K]$, we use the shorthand
\begin{align*}
\eta_k \defn 16 \varplusone \Big( \sqrt{\frac{ |\BL_k| \dimone \dimtwo}{N} \log(\dimone \dimtwo) } + 2 \frac{\dimone \dimtwo}{N} \log(\dimone \dimtwo) \Big)
\end{align*}
throughout the proof. 
Applying Lemma~\ref{lem:par-sum} with $\mathcal{S} = \{i\} \times
[\dimtwo]$ and then with $\mathcal{S} = \{i\} \times \BL_k$ for each $i
\in [\dimone], k \in [K]$, we obtain that
\begin{subequations}
\begin{align} \label{eq:uni-bd-1}
\Pr \left\{ \Big| S(i) -
\sum_{\ell \in [\dimtwo]} M^*_{i,\ell} \Big| \ge \frac{\eta}{2}
\right\} \le 2(\dimone \dimtwo)^{-4} ,
\end{align}
and that
\begin{align} \label{eq:uni-bd-2}
\Pr \left\{ \Big| S_{\BL_k}(i) -
\sum_{\ell \in \BL_k} M^*_{i,\ell} \Big| \ge \frac{\eta_{k}}{2}
\right\} \le 2(\dimone \dimtwo)^{-4} .
\end{align}
\end{subequations}
Note that $K \leq \dimtwo /
\beta \leq \dimtwo^{1/2}$, so a union bound over all $\dimone (K+1)$ events in inequalities~\eqref{eq:uni-bd-1} and~\eqref{eq:uni-bd-2} yields that $\Pr \{\cE\} \ge 1 - 2(\dimone \dimtwo)^{-3}$, where we define the event
\begin{align*}
\cE  \defn \left\{ \Big| S(i) -
\sum_{\ell \in [\dimtwo]} M^*_{i,\ell} \Big| \le \frac{\eta}{2} \text{ and } \Big| S_{\BL_k}(i) -
\sum_{\ell \in \BL_k} M^*_{i,\ell} \Big| \le \frac{\eta_{k}}{2} \text{ for all } i \in [\dimone], k \in [K]
\right\} .
\end{align*}

We now condition on event $\cE$. Applying the triangle inequality yields that if
\begin{align*}
S(v) - S(u) > \eta \quad \text{ or } \quad
S_{\BL_k}(v) - S_{\BL_k}(u) > \eta_k,
\end{align*}
then we have 
\begin{align*}
\sum_{\ell \in [\dimtwo]} M^*_{v,\ell} - \sum_{\ell \in [\dimtwo]} M^*_{u,\ell} > 0 \quad \text{ or } \quad
\sum_{\ell \in \BL_k} M^*_{v,\ell} - \sum_{\ell \in \BL_k} M^*_{u,\ell} > 0 .
\end{align*}
It follows that $u < v$ since $M^*$ has nondecreasing columns. Thus, by the choice of thresholds
$\eta$ and $\eta_k$ in inequalities~\eqref{eq:full-sum} and~\eqref{eq:block-sum}, we have
guaranteed that every edge $u \to v$ in the graph $G$ is consistent
with the underlying permutation $\id$, so a topological sort exists on event $\cE$.

Conversely, if we have
\begin{align*}
\sum_{\ell \in [\dimtwo]} M^*_{v,\ell} - \sum_{\ell \in [\dimtwo]} M^*_{u,\ell} > 2 \eta \quad \text{ or } \quad
\sum_{\ell \in \BL_k} M^*_{v,\ell} - \sum_{\ell \in \BL_k} M^*_{u,\ell} > 2 \eta_k,
\end{align*}
then the triangle inequality implies
that
\begin{align*}
S(v) - S(u) > \eta \quad \text{ or } \quad
S_{\BL_k}(v) - S_{\BL_k}(u) > \eta_k .
\end{align*}
Hence the edge $u \to v$ is present in the graph $G$, so the topological sort $\pihattds(u)$ satisfies the relation $\pihattds(u) < \pihattds(v)$. Claim that this allows us to obtain the following bounds on event $\cE$:
\begin{subequations}
\begin{align}
\Big| \sum_{j \in [\dimtwo]} (M^*_{\pihatftds(i), j} - M^*_{i,j}) \Big| &\le 96 \varplusone \sqrt{\frac{\dimone \dimtwo^2}{N} \log(\dimone \dimtwo)} \quad \text{ for all } i \in [\dimone], \text{ and} \label{eq:per-bd-1} \\
\Big| \sum_{j \in \BL_k} (M^*_{\pihatftds(i), j} - M^*_{i,j}) \Big| &\le 96 \varplusone \sqrt{\frac{\dimone \dimtwo}{N} |\BL_k| \log (\dimone \dimtwo)} \quad \text{ for all } i \in [\dimone], k \in [K] . \label{eq:per-bd-2} 
\end{align}
\end{subequations}

We now prove inequality~\eqref{eq:per-bd-2}. The proof of inequality~\eqref{eq:per-bd-1} follows in the same fashion. We split the proof into two cases.

\paragraph{Case 1.} First, suppose that $|\BL_k| \ge \frac{\dimone \dimtwo}{N}
\log(\dimone \dimtwo)$. Applying Lemma~\ref{lem:per-num} with
$a_i = \sum_{\ell \in \BL_k} M^*_{i,\ell}$, $\pi = \pihattds$ and
$\tau=2\eta_{k}$, we see that for all $i \in [\dimone]$,
\begin{align*}
\Big| \sum_{\ell  \in \BL_k} (M^*_{\pihattds(i),\ell} - M^*_{i,\ell} )
\Big| &\le 2 \eta_{k} \le 96 \varplusone \sqrt{\frac{\dimone \dimtwo}{N} |\BL_k| \log(\dimone \dimtwo)}.
\end{align*}

\paragraph{Case 2.} Otherwise,
we have $|\BL_k| \le \frac{\dimone \dimtwo}{N} \log(\dimone
\dimtwo)$. It then follows that
\begin{align*}
\Big| \sum_{\ell  \in \BL_k} (M^*_{\pihattds(i),\ell} - M^*_{i,\ell} )
\Big| \le 2 |\BL_k| \le 2 \sqrt{\frac{\dimone \dimtwo}{N} |\BL_k|
  \log(\dimone \dimtwo)},
\end{align*}
where we have used the fact that $M \in [0,1]^{\dimone \times \dimtwo}$.

Next, we consider concentration of the column sums of $Y^{(1)}$. Applying Lemma~\ref{lem:par-sum} again with $\mathcal{S} = [\dimone] \times \{j\}$, we obtain that
\begin{align} \label{eq:col-con}
\Big| C(j) - \sum_{i =1}^{\dimone} M^*_{i,j} \Big| \le 8 \varplusone \bigg( \sqrt{\frac{ \dimone^2 \dimtwo}{N} \log(\dimone \dimtwo) } + 2 \frac{\dimone \dimtwo}{N} \log(\dimone \dimtwo) \bigg)
\end{align}
for all $j \in [\dimtwo]$ with probability at least $1 - 2 (\dimone
\dimtwo)^{-3}$.  We carry out the remainder of the proof conditioned
on the event of probability at least $1-4(\dimone \dimtwo)^{-3}$ that
inequalities \eqref{eq:per-bd-1}, \eqref{eq:per-bd-2} and
\eqref{eq:col-con} hold.

Having stated the necessary bounds, we now split the remainder of the
proof into two parts for convenience. In order to do so, we first
split the set $\BL$ into two disjoint sets of blocks, depending on
whether a block comes from an originally large block (of size larger
than $\beta = \blocksize$ as in Step~3 of Subroutine~1) or from an
aggregation of small blocks.  More formally, define the sets
\begin{subequations}
\begin{align*}
\bllarge & \defn \{B \in \BL: B \text{ was not obtained via
  aggregation}\}, \text{ and} \\
\blsmall & \defn \BL \setminus \bllarge.
\end{align*}
\end{subequations}
For a set of blocks $\mathsf{B}$, define the shorthand $\cup
\mathsf{B} = \bigcup_{B \in \mathsf{B}} B$ for convenience.  We begin
by focusing on the blocks $\bllarge$.


\subsubsection{Error on columns indexed by $\cup \bllarge$}

Recall that when the columns of the matrix are ordered according to
$\sigmahatpre$, the blocks in $\bllarge$ are contiguous and thus have
an intrinsic ordering. We index the blocks according to this ordering
as $B_1, B_2, \ldots, B_\ell$ where $\ell = |\bllarge|$. Now define
the disjoint sets
\begin{align*}
\blone & \defn \{ B_k \in \bllarge: k = 0 \ (\modd 2) \}, \text{ and}
\\
\bltwo & \defn \{ B_k \in \bllarge: k = 1 \ (\modd 2) \}.
\end{align*}
Let $\ell_t = |\blt|$ for each $t = 1,2$.

Recall that each block $B_k$ in $\bllarge$ remains unchanged after
aggregation, and that the threshold we used to block the columns is
$\tau = 16 \varplusone \big( \sqrt{\frac{\dimone^2 \dimtwo}{N}
  \log(\dimone \dimtwo) } + 2 \frac{\dimone \dimtwo}{N} \log(\dimone
\dimtwo) \big)$. Hence, applying the concentration bound
\eqref{eq:col-con} together with the definition of blocks in Step~2 of
Subroutine~1 yields
\begin{align}
\label{eq:col-sum-bd}
\Big| \sum_{i =1}^{\dimone} M^*_{i,j_1} - \sum_{i =1}^{\dimone}
M^*_{i,j_2} \Big| \le 96 \varplusone \sqrt{\frac{\dimone^2 \dimtwo}{N}
  \log(\dimone \dimtwo)} \quad \text{ for all } j_1,j_2 \in B_k ,
\end{align}
where we again used the argument leading to claim~\eqref{eq:per-bd-2}
to combine the two terms.  Moreover, since the threshold is twice the
concentration bound, it holds that under the true ordering $\id$,
every index in $B_k$ precedes every index in $B_{k+2}$ for any $k \in
[K-2]$.  By definition, we have thus ensured that the blocks in $\blt$
do not ``mix'' with each other.

The rest of the argument hinges on the following lemma, which is
proved in Section~\ref{sec:per-rc-bd}.

\begin{lemma}
\label{lem:per-rc-bd}
For $m \in \mathbb{Z}_+$, let $J_1 \sqcup \cdots \sqcup J_\ell$ be a
partition of $[m]$ such that each $J_k$ is contiguous and $J_k$
precedes $J_{k+1}$. Let $a_k = \min J_k$, $b_k = \max J_k$ and $m_k =
|J_k|$. Let $A$ be a matrix in $[0,1]^{n\times m}$ with nondecreasing
rows and nondecreasing columns. Suppose that
\begin{align*}
\sum_{i=1}^n (A_{i,b_k} - A_{i,a_k}) \le \tau \ \text{ for each } k
\in [\ell] \text{ and some } \tau \ge 0.
\end{align*} 
Additionally, suppose that there are positive reals $\rho, \rho_1,
\rho_2, \ldots, \rho_\ell$, and a permutation $\pi$ such that for any
$i \in [n]$, we have \1 $\sum_{j=1}^{m} |A_{\pi(i),j} - A_{i,j}| \le
\rho$, and \2 $\sum_{j \in J_k} |A_{\pi(i),j} - A_{i,j}| \le \rho_k$
for each $k \in [\ell]$.  Then it holds that
\begin{align*}
\sum_{i=1}^n \sum_{j=1}^m (A_{\pi(i),j} - A_{i,j})^2 \le 2\tau
\sum_{k=1}^\ell \rho_k + n \rho \max_{k \in [\ell]}
\frac{\rho_k}{m_k}.
\end{align*}
\end{lemma}

We apply the lemma as follows. For $t = 1,2$, let the matrix $M^{(t)}$
be the submatrix of $M^*$ restricted to the columns indexed by the
indices in $\cup \blt$. The matrix $M^{(t)}$ has nondecreasing rows
and columns by assumption. We have shown that the blocks in $\blt$ do
not mix with each other, so they are contiguous and correctly ordered
in $M^{(t)}$. Moreover, the inequality assumptions of the lemma
correspond to \eqref{eq:col-sum-bd}, \eqref{eq:per-bd-1} and
\eqref{eq:per-bd-2} respectively, with the substitutions
\begin{align*}
A = M^{(t)}, \qquad n = \dimone, \qquad m = |\cup \blt|, \qquad \tau =
96 \varplusone \sqrt{\frac{\dimone^2 \dimtwo}{N} \log(\dimone
  \dimtwo)} \\
\rho = 96 \varplusone \sqrt{\frac{\dimone \dimtwo^2}{N} \log(\dimone
  \dimtwo)}, \qquad \rho_k = 96 \varplusone \sqrt{\frac{\dimone
    \dimtwo}{N} |J_k| \log(\dimone \dimtwo)},
\end{align*}
and setting $J_1, \dots, J_\ell$ to be the blocks in $\blt$.
Therefore, applying Lemma~\ref{lem:per-rc-bd} yields
\begin{align*}
& \quad \ \sum_{i \in [\dimone]} \sum_{j \in \cup \blt}
  (M^*_{\pihatftds(i),j} - M^*_{i,j})^2 \\ &\lesssim \maxvarone
  \frac{\dimone^{3/2} \dimtwo}{N} \log (\dimone \dimtwo) \sum_{B \in
    \blt} \sqrt{|B|} + \maxvarone \frac{\dimone^2 \dimtwo^{3/2}}{N}
  \log (\dimone \dimtwo) \max_{B \in \blt} \frac{\sqrt{|B|}}{|B|}
  \\ &\stackrel{\1}{\leq} \maxvarone \frac{\dimone^{3/2} \dimtwo}{N}
  \log (\dimone \dimtwo) \sqrt{\sum_{B \in \blt} |B|} \sqrt{\ell_t} +
  \maxvarone \frac{\dimone^2 \dimtwo^{3/2}}{N} \log (\dimone \dimtwo)
  \frac{1}{\min_{B \in \blt} \sqrt{|B|}} \\ &\stackrel{\2}{\leq}
  \frac{\maxvarone}{\sqrt{\beta}} \frac{\dimone^{3/2} \dimtwo^{2}}{N}
  \log (\dimone \dimtwo) + \frac{\maxvarone}{\sqrt{\beta}}
  \frac{\dimone^2 \dimtwo^{3/2}}{N} \log (\dimone \dimtwo)
  \\ &\lesssim \frac{\maxvarone}{\sqrt{\beta}} \left(\dimone \dimtwo
  \right)^{3/2} \maxdimonetwo^{1/2} \frac{\log (\dimone \dimtwo)}{N},
\end{align*}
where step $\1$ follows from the Cauchy-Schwarz inequality, and step
$\2$ follows from the fact that $\min_{B \in \blt} |B| \ge \beta =
\blocksize$ so that $\ell_t \leq \dimtwo / \beta$. Substituting for
$\beta$ and normalizing by $\dimone \dimtwo$ yields
\begin{align}
\frac{1}{\dimone \dimtwo} \sum_{i \in [\dimone]} \sum_{j \in \cup
  \blt} (M^*_{\pihatftds(i),j} - M^*_{i,j})^2 \lesssim \maxvarone
\dimone^{1/4} \maxdimonetwo^{1/2} \left( \frac{\log (\dimone
  \dimtwo)}{N} \right)^{3/4}. \label{eq:bigblockerror}
\end{align}
This proves the required result for the set of blocks
$\BL^{(t)}$. Summing over $t = 1, 2$ then yields a bound of twice the
size for columns of the matrix indexed by $\cup \bllarge$.

\subsubsection{Error on columns indexed by $\cup \blsmall$}

Next we bound the approximation error of each row of the matrix with
column indices restricted to the union of all small blocks.  In the
easy case where $\blsmall$ contains a single block of size less than
$\frac 12 \blocksize$, we have
\begin{align*}
\sum_{i \in [\dimone]} \sum_{j \in \cup \blsmall} (M^*_{\pihatftds(i),j} - M^*_{i,j})^2 
& \stackrel{\1}{\le} \sum_{i \in [\dimone]}  \sum_{j \in \cup \blsmall} \big| M^*_{\pihatftds(i),j} - M^*_{i,j} \big| \\
& \stackrel{\2}{=} \sum_{i \in [\dimone]}  \Big| \sum_{j \in \cup \blsmall} (M^*_{\pihatftds(i),j} - M^*_{i,j}) \Big| \\
&\stackrel{\3}{\le} \sum_{i \in [\dimone]} 96 \varplusone \sqrt{ \frac{\dimone \dimtwo}{2N} \dimtwo \big[\frac{\maxdimonetwo}{N}\big]^{1/2} \log^{3/2} (\dimone \dimtwo)  } \\
&=  48 \sqrt{2} \varplusone \frac{\dimone^{3/2} \dimtwo \maxdimonetwo^{1/4}}{N^{3/4}} \log^{3/4} (\dimone \dimtwo) ,
\end{align*}
where step \1 follows from the H\"older's inequality and the fact that $M^* \in [0,1]^{\dimone \times \dimtwo}$, step \2 from the monotonicity of the columns of $M^*$, and step \3 from equation~\eqref{eq:per-bd-1}.

Now we aim to prove a bound of the same order for the general case.
Critical to our analysis is the following lemma:

\begin{lemma} \label{lem:l2-tv-l1}
For a vector $v \in \mathbb{R}^n$, define its variation as $\var(v) =
\max_i v_i - \min_i v_i$. Then we have
\begin{align*}
\|v\|_2^2 \le \var(v) \|v\|_1 + \|v\|_1^2/n.
\end{align*}
\end{lemma}
\noindent See Section~\ref{sec:l2-tv-l1} for the proof of this claim.

For each $i \in [\dimone]$, define $\Delta^i$ to be the restriction of
the $i$-th row difference $M^*_{\pihatftds(i)} - M^*_{i}$ to the union
of blocks $\cup \blsmall$. For each block $B \in \blsmall$, denote the
restriction of $\Delta^i$ to $B$ by
$\Delta^i_B$. Lemma~\ref{lem:l2-tv-l1} applied with $v = \Delta^i_{B}$
yields
\begin{align}
\|\Delta^i\|_2^2 &= \sum_{B \in \blsmall} \| \Delta^i_B \|_2^2 \notag \\
&\leq \sum_{B \in \blsmall} \var(\Delta^i_B) \|\Delta^i_B \|_1 + \sum_{B \in \blsmall} \frac{\|\Delta^i_B\|_1^2}{|B|} \notag \\
&\leq \left( \max_{B \in \blsmall} \|\Delta^i_B \|_1 \right) \sum_{B \in \blsmall} \var\left(\Delta^i_B \right) +  \frac{\max_{B \in \blsmall} \|\Delta^i_B\|_1}{\min_{B \in \blsmall} |B|} \sum_{B \in \blsmall}\|\Delta^i_B\|_1 \notag \\
&\le \left( \max_{B \in \blsmall} \|\Delta^i_B \|_1 \right) \left( \sum_{B \in \blsmall} \var\left(\Delta^i_B \right) \right) +  \frac{\max_{B \in \blsmall} \|\Delta^i_B\|_1}{\min_{B \in \blsmall} |B|} \sum_{B \in \blsmall}\|\Delta^i_B\|_1 . \label{eq:row-l2-bd}
\end{align}
We now analyze the quantities in inequality~\eqref{eq:row-l2-bd}.
By the aggregation step of Subroutine~1, we have $\frac{1}{2} \beta \leq |B| \leq 2 \beta$, where $\beta = \blocksize$. Additionally, the bounds \eqref{eq:per-bd-1} and \eqref{eq:per-bd-2} imply that
\begin{align*}
\sum_{B \in \blsmall} \|\Delta^i_B\|_1 &= \|\Delta^i \|_1 \le 96 \varplusone \sqrt{\frac{\dimone \dimtwo^2}{N} \log (\dimone \dimtwo)} \lesssim \varplusone \beta, \ \text{ and } 
\\
\|\Delta^i_B \|_1 &\leq 96 \varplusone \sqrt{\frac{\dimone \dimtwo}{N} |B| \log (\dimone \dimtwo)} \\
&\leq 96 \sqrt{2} \varplusone \sqrt{\frac{\dimone \dimtwo}{N} \beta \log (\dimone \dimtwo)} \ \text{ for all } B \in \blsmall.
\end{align*}

Moreover, to bound the quantity $\sum_{B \in \blsmall} \var\left(\Delta^i_B \right)$, we proceed as in the proof for the large blocks in $\bllarge$. Recall that if we permute the columns by $\sigmahatpre$ according to the column sums, then the blocks in $\blsmall$ have an intrinsic ordering, even after adjacent small blocks are aggregated. Let us index the blocks in $\blsmall$ by $B_1, B_2, \ldots, B_m$ according to this ordering, where $m = |\blsmall|$. As before, the odd-indexed (or even-indexed) blocks do not mix with each other under the true ordering $\id$, because the threshold used to define the blocks is larger than twice the column sum perturbation.
We thus have
\begin{align*}
\sum_{B \in \blsmall} \var\left(\Delta^i_B \right) & = \sum_{\substack{k \in [m] \\ k \text{ odd}}} \var(\Delta^i_{B_k} ) + \sum_{\substack{k \in [m] \\ k \text{ even}}} \var(\Delta^i_{B_k} ) \\
&\le \sum_{\substack{k \in [m] \\ k \text{ odd}}} \big[ \var(M^*_{i, B_k} ) + \var(M^*_{\pihatftds(i), B_k} ) \big] + \sum_{\substack{k \in [m] \\ k \text{ even}}} \big[ \var(M^*_{i, B_k} ) + \var(M^*_{\pihatftds(i), B_k} ) \big] \\
& \stackrel{\1}{\le} 2\var(M^*_i) + 2\var(M^*_{\pihatftds(i)}) \stackrel{\2}{\le} 4 ,
\end{align*}
where inequality \1 holds because the odd (or even) blocks do not mix, and inequality \2 holds because $M^*$ has monotone rows in $[0,1]^{\dimtwo}$. 

Finally, putting together all the pieces, we can substitute for $\beta$, sum over the indices $i \in \dimone$, and normalize by $\dimone \dimtwo$ to obtain
\begin{align}
\frac{1}{\dimone \dimtwo} \sum_{i \in [\dimone]} \|\Delta^i\|_2^2 \lesssim \maxvarone \left( \frac{\dimone \log (\dimone \dimtwo)}{N} \right)^{3/4}, \label{eq:smallblockerror}
\end{align}
and so the error on columns indexed by the set $\cup \blsmall$ is bounded as desired.


\medskip

Combining the bounds~\eqref{eq:bigblockerror}
and~\eqref{eq:smallblockerror}, we conclude that
\begin{align*}
\frac{1}{\dimone \dimtwo} \|M^*(\pihatftds, \id) - M^*\|_F^2 \lesssim \maxvarone \dimone^{1/4} \maxdimonetwo^{1/2} \left( \frac{\log (\dimone \dimtwo)}{N} \right)^{3/4} 
\end{align*}
with probability at least $1 - 4(\dimone \dimtwo)^{-3}$.
The same proof works with the roles of $\dimone$ and $\dimtwo$ switched and all the matrices transposed, so it holds with the same probability that
\begin{align*}
\frac{1}{\dimone \dimtwo} \|M^*(\id, \sigmahatftds) - M^*\|_F^2 \lesssim \maxvarone \dimtwo^{1/4} \maxdimonetwo^{1/2} \left( \frac{\log (\dimone \dimtwo)}{N} \right)^{3/4}.
\end{align*}
Consequently,
\begin{align*}
\frac 1{\dimone \dimtwo} \left( \|M^*(\pihatftds, \id) - M^*\|_F^2 + \|M^*(\id, \sigmahatftds) - M^*\|_F^2 \right) \lesssim \maxvarone \Big(\frac{ \dimone \log \dimone}{N} \Big)^{3/4} 
\end{align*}
with probability at least $1 - 8(\dimone \dimtwo)^{-3}$, where we have
used the relation $\dimone \geq \dimtwo$. Applying
Proposition~\ref{prop:meta} completes the proof.


\subsubsection{Proof of Lemma~\ref{lem:per-rc-bd}}
\label{sec:per-rc-bd}

Since $A$ has increasing rows, for any $i,i_2 \in [n]$ with $i \le
i_2$ and any $j, j_2 \in J_k$, we have
\begin{align*}
A_{i_2,j} - A_{i,j} &= (A_{i_2,j} - A_{i_2,a_k}) + (A_{i_2,a_k} - A_{i,b_k}) + (A_{i,b_k} - A_{i,j}) \\
&\le (A_{i_2,b_k} - A_{i_2,a_k}) + (A_{i_2,j_2} - A_{i,j_2}) + (A_{i,b_k} - A_{i,a_k}) .
\end{align*}
Choosing $j_2 = \arg \min_{r \in J_k} (A_{i_2,r} - A_{i,r})$, we obtain
\begin{align*}
A_{i_2,j} - A_{i,j} \le (A_{i_2,b_k} - A_{i_2,a_k}) + (A_{i,b_k} - A_{i,a_k}) + \frac 1{m_k} \sum_{r \in J_k} (A_{i_2,r} - A_{i,r}) .
\end{align*}
Together with the assumption on $\pi$, this implies that
\begin{align*}
|A_{\pi(i),j} - A_{i,j}| &\le \underbrace{A_{\pi(i),b_k} - A_{\pi(i),a_k}}_{ =: \, x_{i,k}} + \underbrace{A_{i,b_k} - A_{i,a_k}}_{ =: \, y_{i,k}} + \frac 1{m_k} \underbrace{\sum_{r \in J_k} |A_{\pi(i),r} - A_{i,r}|}_{ =: \, z_{i,k}}.
\end{align*}
Hence it follows that
\begin{align*}
\sum_{i=1}^n \sum_{j=1}^m (A_{i,j} - A_{\pi(i),j})^2 
&= \sum_{i=1}^n \sum_{k=1}^\ell \sum_{j \in J_k} (A_{i,j} - A_{\pi(i),j})^2  \\
&\le \sum_{i=1}^n \sum_{k=1}^\ell \sum_{j \in J_k} |A_{i,j} - A_{\pi(i),j}| ( x_{i,k} + y_{i,k} + z_{i,k}/m_k ) \\
&= \sum_{i=1}^n \sum_{k=1}^\ell z_{i,k} ( x_{i,k} + y_{i,k} + z_{i,k}/m_k ) .
\end{align*}
According to the assumptions, we have
\begin{enumerate}
\item $\sum_{k=1}^\ell x_{i,k} \le 1$ and $\sum_{i=1}^n x_{i,k} \le \tau$ for any $i \in [n], k \in [\ell]$;
\item $\sum_{k=1}^\ell y_{i,k} \le 1$ and $\sum_{i=1}^n y_{i,k} \le \tau$ for any $i \in [n], k \in [\ell]$;
\item $z_{i,k} \le \rho_k$ and $\sum_{k=1}^\ell z_{i,k} \le \rho$ for any $i \in [n], k \in [\ell]$.
\end{enumerate}
Consequently, the following bounds hold:
\begin{enumerate}
\item $\sum_{i=1}^n \sum_{k=1}^\ell z_{i,k} x_{i,k} \le \sum_{i=1}^n \sum_{k=1}^\ell \rho_k x_{i,k} \le \tau \sum_{k=1}^\ell \rho_k$;
\item $\sum_{i=1}^n \sum_{k=1}^\ell z_{i,k} y_{i,k} \le \sum_{i=1}^n \sum_{k=1}^\ell \rho_k y_{i,k} \le \tau \sum_{k=1}^\ell \rho_k$;
\item $\sum_{i=1}^n \sum_{k=1}^\ell z_{i,k}^2/m_k \le \sum_{i=1}^n \sum_{k = 1}^\ell z_{i,k} \cdot \max_{k \in [\ell]} (\rho_k/m_k) \le n \rho \max_{k \in [\ell]} (\rho_k/m_k)$.
\end{enumerate}
Combining these inequalities yields the claim.

\subsubsection{Proof of Lemma~\ref{lem:l2-tv-l1}} \label{sec:l2-tv-l1}

Let $a = \min_{i\in [n]} v_i$ and $b = \max_{i\in [n]} v_i = a+ \var(v)$.
Since the quantities in the inequality remain the same if we replace $v$ by $-v$, we assume without loss of generality that $b \ge 0$. If $a \le 0$, then $\|v\|_\infty \le b-a = \var(v)$. If $a > 0$, then $a \le \|v\|_1/n$ and $\|v\|_\infty = b \le \|v\|_1/n + \var(v)$. Hence in any case we have $\|v\|_2^2 \le \|v\|_\infty \|v\|_1 \le [\|v\|_1/n + \var(v)] \|v\|_1$.


\section{Discussion}

While the current paper narrows the statistical-computational gap for
estimation in permutation-based models with monotonicity constraints,
several intriguing questions remain:
\begin{itemize}
\item Can Algorithm~2 be recursed so as to improve the rate of
  estimation, until we eventually achieve the statistically optimal
  rate (up to lower-order terms) in polynomial time?
\item If not, does there exist a statistical-computational gap in this
  problem, and if so, what is the fastest rate achievable by
  computationally efficient estimators?
\item Can the techniques from here be used to narrow
  statistical-computational gaps in other permutation-based
  models~\citep{ShaBalWai16,FlaMaoRig16,PanWaiCou17}?
\end{itemize}

As a partial answer to the first question, it can be shown that when our two-dimensional sorting algorithm is recursed in the natural way and applied to
the noisy sorting subclass of the SST model, it yields another minimax
optimal estimator for noisy sorting, similar to the multistage
algorithm of Mao et al.~\cite{MaoWeeRig17}. However, showing that this
same guarantee is preserved for the larger class of SST matrices seems out of the reach of techniques introduced in this paper. In fact, we conjecture that any algorithm that only exploits partial row and column sums cannot achieve a rate faster than $O(n^{-3/4})$ for the SST class.

It is also worth noting that the model~\eqref{eq:model} allowed us to
perform multiple sample-splitting steps while preserving the
independence across observations. While our proofs also hold for the
observation model where we have exactly $3$ independent samples per
entry of the matrix, handling the weak dependence of the
sampling model with one observation per entry is an interesting
technical challenge that may also involve its own
statistical-computational tradeoffs~\citep{Mon15}.


\section*{Acknowledgments}

CM thanks Philippe Rigollet for helpful discussions. The work of CM
was supported in part by grants NSF CAREER DMS-1541099, NSF DMS-1541100 and
ONR N00014-16-S-BA10, and the work of AP and MJW was supported in part
by grants NSF-DMS-1612948 and DOD ONR-N00014. We thank Jingyan Wang for pointing out an error in an earlier version of the paper.


\appendix


\section{Proof of Lemma~\ref{lem:shah}} \label{app:ls}

The proof parallels that of Shah et
al.~\citet[Theorem~5(a)]{ShaBalGunWai17}, so we only emphasize the
differences and sketch the remaining argument. We may assume that
$\pobs \ge \frac{1}{\dimtwo}$, since otherwise the bound is trivial.

We first employ a truncation argument. Consider the event
\begin{align*}
\cE \defn \left\{ |W_{i,j}| \le \frac{c_3}{\pobs} (\vars \lor 1)
\sqrt{\log (\dimone \dimtwo)} \text{ for all } i \in [\dimone], j \in
     [\dimtwo] \right\} .
\end{align*}
If the universal constant $c_3$ is chosen to be sufficiently large,
then it follows from the sub-Gaussianity of $W_{i,j}$ and a union
bound over all index pairs $(i, j) \in [\dimone] \times [\dimtwo]$ that
$\Pr\{\cE\} \ge 1 - (\dimone \dimtwo)^{-4}$.  
Now define the
truncation operator
\begin{align}
\label{eq:truncate}
T_{\lambda}(x) & \defn
\begin{cases}
x & \text{ if } |x| \leq \lambda, \\
\lambda \cdot \sgn(x) & \text{ otherwise.} 
\end{cases}
\end{align}
With the choice $\lambda = \frac{c_3}{\pobs} (\vars \lor 1) \sqrt{\log
  (\dimone \dimtwo)}$, define the random variables $W^{(1)}_{i,j} =
T_{\lambda} \left(W_{i,j}\right)$ for each pair of indices $(i, j) \in [\dimone] \times [\dimtwo]$.  Consider the model where we observe
$M^*+W^{(1)}$ instead of $Y = M^*+W$. Then the new model and the
original one are coupled so that they coincide on the event
$\cE$. Therefore, it suffices to 
prove a high probability bound assuming that the noise is given by
$W^{(1)}$.

Let us define $\mu = \EE[W^{(1)}]$ and $\Wtil = W^{(1)} - \mu$. We
claim that for any $i \in [\dimone], j \in [\dimtwo]$, the following
relations hold:
\begin{enumerate}
\item $|\mu_{i,j}| \le \frac{c}{\pobs} (\vars \lor 1) (\dimone
  \dimtwo)^{-4}$; \label{clm:1}
\item $\Wtil_{i,j}$ are independent, centered and $\frac{c}{\pobs}
  (\vars \lor 1)$-sub-Gaussian; \label{clm:2}
\item $|\Wtil_{i,j}| \le \frac{c}{\pobs} (\vars \lor 1) \sqrt{\log (\dimone \dimtwo)}$; \label{clm:3}
\item $\EE[|\Wtil_{i,j}|^2] \le \frac{c}{\pobs}
  \maxvarone$. \label{clm:4}
\end{enumerate} 
Taking these claims as given for the moment, we turn to the main
argument assuming that our observations take the form $Y = M^* + \Wtil
+ \mu$.

For any permutations $\pi \in \symgp_{\dimone}, \sigma \in
\symgp_{\dimtwo}$, let $\Mps = \Mhatls(Y)$.  We
claim that for any fixed pair $(\pi, \sigma)$ such that $ \|Y -
\Mps\|_F^2 \le \|Y - M^*\|_F^2$, we have
\begin{align}
\label{eq:fixed-perm-bd}  
\Pr \Big\{ \|\Mps - M^*\|_F^2 \ge c_1 \maxvarone \frac{\dimone}{\pobs}
\log^2 (\dimone) \Big\} \le \dimone^{-3
  \dimone}. 
\end{align}
Treating claim~\eqref{eq:fixed-perm-bd} as true for the moment, we see
that since the least squares estimator $\Mhat$ is equal to $\Mps$ for
some pair $(\pi,\sigma)$, a union bound over $\pi \in
\symgp_{\dimone}, \sigma \in \symgp_{\dimtwo}$ yields
\begin{align*}
\Pr \left\{ \|\Mhat - M^*\|_F^2 \ge c_1 \maxvarone
\frac{\dimone}{\pobs} \log^2 \dimone \right\} \le \dimone^{-\dimone},
\end{align*}
which completes the proof. Thus, to prove our result, it suffices to
prove claim~\eqref{eq:fixed-perm-bd}.

Let $\Dps = \Mps - M^*$. The condition $\|Y - \Mps\|_F^2 \le \|Y -
M^*\|_F^2$ yields the basic inequality
\begin{align*}
\frac 12 \|\Dps\|_F^2 \le \langle \Dps, \Wtil + \mu \rangle.
\end{align*}
Since $\Dps \in [-1,1]^{\dimone \times \dimtwo}$, we have $\langle
\Dps, \mu \rangle \le \|\mu\|_1 \le \frac{c}{\pobs} (\vars \lor 1)
\dimone^{-6}$ by claim~\ref{clm:1}. If it holds that $\|\Dps\|_F^2 \le
\frac{4 c}{\pobs} (\vars \lor 1) \dimone^{-6}$, then the proof is
immediate. Thus, we may assume the opposite, from which it follows
that
\begin{align}
\label{eq:basic-1}
\frac 14 \|\Dps\|_F^2 \le \langle \Dps, \Wtil \rangle.
\end{align}

Consider the set of matrices
\begin{align*}
\Cdiffps & \defn \left\{ \alpha (M-M^*) \, : \, M \in \Cbiso(\pi,
\sigma), \, \alpha \in [0,1] \right\} .
\end{align*}
Additionally, for every $t>0$, define the random variable
\begin{align*}
\Zps(t) & \defn \sup_{\substack{D \in \Cdiffps, \\ \|D\|_F \le t}}
\langle D, \Wtil \rangle .
\end{align*}
For every $t > 0$, define the event
\begin{align*}
\At & \defn \left\{ \text{there exists } D \in \Cdiffps \text{ such
  that } \|D\|_F \ge \sqrt{t \delta_n} \text{ and } \langle D, \Wtil
\rangle \ge 4 \|D\|_F \sqrt{t \delta_n} \right\} .
\end{align*}
For $t \ge \delta_n$, either we already have $\|\Dps\|_F^2 \le t
\delta_n$, or we have $\|\Dps\|_F > \sqrt{t \delta_n}$. In the latter
case, on the complement of $\At$, we must have $\langle \Dps, \Wtil
\rangle \le 4 \|\Dps\|_F \sqrt{t \delta_n}$. Combining this with
inequality \eqref{eq:basic-1} then yields $\|\Dps\|_F^2 \le c t
\delta_n$.  It thus remains to bound the probability $\Pr\{\At\}$.

Using the star-shaped nature of the set $\Cdiffps$, a rescaling
argument yields
\begin{align*}
\Pr \{ \At \} \le \Pr \left\{ \Zps(\delta_n) \ge 4 \delta_n \sqrt{t
  \delta_n} \right\} \quad \text{ for all } t \ge \delta_n .
\end{align*}
The following lemma bounds the tail behavior of the random variable
$\Zps(\delta_n)$, and its proof is postponed to
Section~\ref{sec:zn-tail}.

\begin{lemma}
\label{lem:zn-tail}
For any $\delta > 0$ and $u > 0$, we have
\begin{align*}
\Pr \left\{ \Ztil(\delta) > \frac{c}{\pobs} (\vars \lor 1) \sqrt{\log
  \dimone} \left( \dimone \log^{1.5} n + u \right) \right\} \le \exp
\left( \frac{- c_1 u^2}{\pobs \delta^2 / (\log \dimone) + \dimone
  \log^{1.5} \dimone + u} \right) .
\end{align*}
\end{lemma}

Taking the lemma as given and setting $\delta_n^2 = \frac{c_2}{\pobs}
\maxvarone \dimone \log^2 \dimone$ and $u = c_3 (\vars \lor 1) \dimone
\log^{1.5} \dimone$, we see that for any $t \ge \delta_n$, we have
\begin{align}
\Pr \{ \At \} \le \Pr \left\{ \Zps(\delta_n) \ge 4 \delta_n \sqrt{t
  \delta_n} \right\} \le \exp \left( \frac{- c_4 \maxvarone \dimone^2
  \log^3 \dimone}{\maxvarone \dimone \log \dimone + \dimone \log^{1.5}
  \dimone} \right) \le \dimone^{-3 \dimtwo}. \label{eq:removelog}
\end{align}
In particular, for $t = \delta_n$, on the complement of $\At$, we have
\begin{align*}
\|\Dps\|_F^2 & \le \frac{c_5}{\pobs} \maxvarone \dimone \log^2 \dimone ,
\end{align*}
which completes the proof. Note that the original proof sacrificed a
logarithmic factor in proving the equivalent of
equation~\eqref{eq:removelog}, and this is why we recover the same
logarithmic factors as in the bounded case in spite of the
sub-Gaussian truncation argument.

In the setting where we know that $M^* \in \Cbiso$, the same proof
clearly works, except that we do not even need to take a union bound
over $\pi \in \symgp_{\dimone}, \sigma \in \symgp_{\dimtwo}$ as the
columns and rows are ordered.


\subsection{Proof of claims 1--4}

We assume throughout that the constant $c_3$ is chosen to be
sufficiently large. Claim~\ref{clm:1} follows as a result of the
following argument; we have
\begin{align*}
\left|\mu_{i, j}\right| &= \left| \EE [W^{(1)}_{i, j}] \right|
\\ &\leq \EE \left[ | W^{(1)}_{i, j} - W_{i, j} | \right] \\ &=
\int_0^\infty \Pr\{ | W^{(1)}_{i, j} - W_{i, j} | \geq t\} dt \\ &=
\int_0^\infty \Pr\{ |W_{i, j}| \geq \frac{c_3}{\pobs} (\vars \lor 1)
\sqrt{\log (\dimone \dimtwo)} + t\} dt \\ &\leq (\dimone \dimtwo)^{-5}
\int_0^\infty \exp \left(\frac{-t^2}{c_4 \maxvarone / \pobs^2}\right)
dt \\ &\leq \frac{c_5}{\pobs} (\vars \lor 1) (\dimone \dimtwo)^{-4}.
\end{align*}

By definition, the random variables $W^{(1)}_{i, j} - \mu_{i, j}$ are
independent and zero-mean, and applying Lemma~\ref{lem:truncate} (see
Appendix~\ref{app:truncate}) yields that they are also sub-Gaussian
with the claimed variance parameter, thus yielding claim~\ref{clm:2}.
The triangle inequality together with the definition of $\Wtil_{i,j}$
then yields claim~\ref{clm:3}.

Finally, since $|T(x)| \leq |x|$, we have
\begin{align*}
\E[|\Wtil_{i,j}|^2] \le \E[|W^{(1)}_{i,j}|^2] \le \E[|W_{i,j}|^2] \le
\frac{c_6}{\pobs} \maxvarone,
\end{align*}
yielding claim~\ref{clm:4}.


\subsection{Proof of Lemma~\ref{lem:zn-tail}}
\label{sec:zn-tail}

The chaining argument from the proof of Shah et
al.~\citet[Lemma~10]{ShaBalGunWai17} can applied to show that
\begin{align*}
  \EE[ \Ztil(\delta) ] \le \frac{ c_2}{\pobs} (\vars \lor 1) \dimone
  \log^2 \dimone,
\end{align*}
as $\Wtil_{i,j}$ is $\frac{c}{\pobs} (\vars \lor 1)$-sub-Gaussian by
claim~\ref{clm:2}.  Note that although we are considering a set of
rectangular matrices $\Cdiffps \subset [-1,1]^{\dimone \times
  \dimtwo}$ instead of square matrices as in \citet{ShaBalGunWai17},
we can augment each matrix by zeros to obtain an $\dimone \times
\dimone$ matrix, and so $\Cdiffps$ can be viewed as a subset of its
counterpart consisting of $\dimone \times \dimone$ matrices. Hence the
entropy bound depending on $\dimone$ can be employed so that the
chaining argument indeed goes through.

In order to obtain the deviation bound, we apply Lemma~11 of Shah et
al.~\citet{ShaBalGunWai17} (i.e., Theorem~1.1(c) of Klein and
Rio~\citet{KleRio05}) with $\mathcal{V} = \Cdiffps \cap
\mathcal{B}_\delta$, $m = \dimone \dimtwo$, $X = \frac{\pobs}{c (\vars
  \lor 1) \sqrt{\log \dimone}} \Wtil$ and $X^\dag = \frac{\pobs}{c
  (\vars \lor 1) \sqrt{\log \dimone}}
\Ztil(\delta)$. Claim~\ref{clm:3} guarantees that $|X|$ is uniformly
bounded by $1$. We also have $\EE[\langle D, \Wtil \rangle^2] \le
\frac{c}{\pobs} \maxvarone \delta^2$ by claim~\ref{clm:4} for
$\|D\|_F^2 \le \delta^2$. Therefore, we conclude that
\begin{align*}
\Pr \left\{ \Ztil(\delta) > \EE[\Ztil(\delta)] + \frac{c}{\pobs}
(\vars \lor 1) \sqrt{\log \dimone} \cdot u \right\} \le \exp \left(
\frac{- c_1 u^2}{\pobs \delta^2 / (\log \dimone) + \dimone \log^{1.5}
  \dimone + u} \right) .
\end{align*}
Combining the expectation and the deviation bounds completes the
proof.


\section{Poissonization reduction} \label{app:poi}

In this section, we show that Poissonization only affects the rates of estimation up to a constant factor. Note that we may assume that
$N \ge 4 \log(\dimone \dimtwo)$, since otherwise, all the bounds in
the theorems hold trivially.

Let us first show that an estimator designed for a Poisson number of samples may be employed for estimation with a fixed number of samples. Assume that $N$ is fixed, and we have an estimator $\Mhat_{\Poi}(N)$, which is designed under $N' = \Poi(N)$ observations $\{y_{\ell}\}_{\ell = 1}^{N'}$. 
Now, given exactly $N$ observations $\{y_{\ell} \}_{\ell = 1}^{N}$ from the model~\eqref{eq:model}, choose an integer $\widetilde{N} = \Poi(N/2)$, and output the estimator
\begin{align*}
\Mhat(N) = 
\begin{cases}
\Mhat_{\Poi} (N/2) & \text{ if } \widetilde{N} \leq N , \\
0 & \text{ otherwise.}
\end{cases}
\end{align*}

Recalling the assumption $N \geq 4 \log (\dimone \dimtwo)$, we have
\begin{align*}
\Pr \{\widetilde{N} \geq N \} \leq e^{-N/2} \leq (\dimone \dimtwo)^{-2}.
\end{align*}
Thus, the error of the estimator $\Mhat(N)$, which always uses at most $N$ samples, is bounded by $\frac{1}{\dimone \dimtwo} \| \Mhat_{\Poi}(N/2) - M^* \|_F^2$ with probability greater than $1 - (\dimone \dimtwo)^{-2}$, and moreover, we have
\begin{align*}
\EE \left[ \frac{1}{\dimone \dimtwo} \| \Mhat(N) - M^* \|_F^2 \right] \leq \EE \left[ \frac{1}{\dimone \dimtwo} \| \Mhat_{\Poi}(N/2) - M^* \|_F^2 \right] + (\dimone \dimtwo)^{-2}.
\end{align*}

We now show the reverse, that an estimator $\Mhat(N)$ designed using exactly $N$ samples may be used to estimate $M^*$ under a Poissonized observation model. 
Given $\widetilde{N} = \Poi(2N)$ samples, define the estimator
\begin{align*}
\Mhat_{\Poi}(2N) =
\begin{cases}
\Mhat(N) & \text{ if } \widetilde{N} \geq N, \\
0 & \text{ otherwise,}
\end{cases} 
\end{align*}
where in the former case, $\Mhat(N)$ is computed by discarding $\widetilde{N} - N$ samples at random.

Again, using the fact that $N \geq 4 \log(\dimone \dimtwo)$ yields 
\begin{align*}
\Pr \{ \widetilde{N} \ge N \} \leq e^{-N} \leq (\dimone \dimtwo)^{-4},
\end{align*}
and so once again, the error of the estimator $\Mhat_{\Poi}(2N)$ is bounded by $\frac{1}{\dimone \dimtwo} \| \Mhat(N) - M^* \|_F^2$ with probability greater than $1 - (\dimone \dimtwo)^{-4}$. A similar guarantee also holds in expectation.


\section{Truncation preserves sub-Gaussianity}
\label{app:truncate}

In this appendix, we show that truncating a sub-Gaussian random
variable preserves its sub-Gaussianity to within a constant factor.

\begin{lemma} \label{lem:truncate}
Let $X$ be a (not necessarily centered) $\sigma$-sub-Gaussian random
variable, and for some choice $\lambda \geq 0$, let $T_{\lambda}(X)$
denote its truncation according to equation~\eqref{eq:truncate}.  Then
$T_{\lambda}(X)$ is $\sqrt{2}\sigma$-sub-Gaussian.
\end{lemma}

\begin{proof}
The proof follows a symmetrization argument. Let $X'$ denote an i.i.d. copy of $X$, and use the shorthand $Y = T_{\lambda} (X)$ and $Y' = T_{\lambda}(X')$. Let $\varepsilon$ denote a Rademacher random variable that is independent of everything else. Then $Y$ and $Y'$ are i.i.d., and $\varepsilon (Y - Y') \stackrel{d}{=} Y - Y'$. Hence we have
\begin{align*}
\EE \left[ e^{t (Y - \EE [Y]) }\right] &= \EE \left[ e^{t (Y - \EE [Y']) }\right] \\
&\leq \EE_{Y, Y'} \left[ e^{t (Y - Y') }\right] \\
&= \EE_{Y, Y', \varepsilon} \left[ e^{t \varepsilon (Y - Y') }\right].
\end{align*}
Using the Taylor expansion of $e^x$, we have
\begin{align*}
\EE \left[ e^{t (Y - \EE [Y]) }\right] &\leq \EE_{Y, Y', \varepsilon} \left[ \sum_{i \geq 0} \frac{1}{i!}\left(t \varepsilon (Y - Y')\right)^i \right] \\
&= \EE_{Y, Y'} \left[ \sum_{j \geq 0} \frac{1}{(2j)!}\left(t (Y - Y')\right)^{2j} \right],
\end{align*} 
since only the even moments remain. Finally, since the map $T_{\lambda}: \mathbb{R} \to \mathbb{R}$ is $1$-Lipschitz, we have $|Y - Y'| \leq |X - X'|$, and combining this with the fact that $X - X'$ has odd moments equal to zero yields
\begin{align*}
\EE \left[ e^{t (Y - \EE [Y]) }\right] &\leq \EE_{X, X'} \left[ \sum_{j \geq 0} \frac{1}{(2j)!}\left(t (X - X')\right)^{2j} \right] \\
&= \EE_{X, X'} \left[ \sum_{i \geq 0} \frac{1}{i!}\left(t (X - X')\right)^{i} \right] \\
&= \EE_{X, X'} \left[ e^{t (X - X') }\right] \\
&\leq e^{t^2 \sigma^2},
\end{align*}
where the last step follows since the random variable $X - X'$ is zero-mean and $\sqrt{2} \sigma$-sub-Gaussian.
\end{proof}
\bibliographystyle{alpha}
\bibliography{fast_sst}

\end{document}